%% file: main.tex
\documentclass[10pt,a4paper]{main}

\usepackage[authoryear,round,sort]{natbib}
\setcitestyle{authoryear,round,citesep={;},aysep={,},yysep={;}}
\input{math_commands.tex}

\usepackage{booktabs}
\usepackage{multirow}
\usepackage{enumitem}
\usepackage{graphicx}
\usepackage{subcaption}
\usepackage{float}
\usepackage{wrapfig}
\usepackage{fontawesome5}

\usepackage[table]{xcolor}
\usepackage[skins]{tcolorbox}
\usepackage{hyperref}
\usepackage{url}
\usepackage{cleveref}

\crefname{proposition}{proposition}{propositions}
\Crefname{proposition}{Proposition}{Propositions}

\definecolor{title_blue}{HTML}{204899}
\definecolor{cite_blue}{HTML}{044DC1}
\definecolor{cite_purple}{HTML}{7406A7}
\hypersetup{
  colorlinks=true,
  citecolor=cite_blue,
  linkcolor=cite_blue,
  urlcolor=cite_blue
}

\definecolor{eapoteal}{HTML}{16866C}
\definecolor{eapored}{HTML}{BE3C2D}
\definecolor{eapopurple}{HTML}{5E338B}
\definecolor{eapoblue}{HTML}{3777BD}
\colorlet{success}{eapoteal}
\colorlet{failure}{eapored}
\colorlet{ours}{eapoblue}
\newcommand{\EAPO}{{\textup{EAPO}}}
\newcommand{\surprisingsuccess}{\emph{\textbf{\textcolor{success}{surprising success}}}}
\newcommand{\repeatedfailure}{\emph{\textbf{\textcolor{failure}{repeated failure}}}}

\colorlet{stdcolor}{black!50}
\newlength{\meanwidth}
\newcommand{\ms}[2]{\settowidth{\meanwidth}{\textbf{00.0}}\makebox[\meanwidth][r]{#1}{\scriptsize\,\textcolor{stdcolor}{$\pm$#2}}}

\definecolor{qualitativemuted}{HTML}{62676C}
\definecolor{qualitativerule}{HTML}{C6C9CC}
\definecolor{qualitativeframe}{HTML}{555B61}
\definecolor{qualitativeproblem}{HTML}{E8EAEC}
\newtcolorbox{qualitativebox}[1][]{
  enhanced, colback=white, colframe=qualitativerule, boxrule=0.55pt,
  arc=2.5pt, outer arc=2.5pt, boxsep=0pt,
  left=8pt, right=8pt, top=7pt, bottom=7pt,
  before skip=6pt, after skip=0pt,
  #1
}
\newenvironment{qualitativerow}{%
  \begin{qualitativebox}[frame hidden,boxrule=0pt,sharp corners,top=8pt,bottom=8pt,before skip=0pt]%
}{\end{qualitativebox}}
\newcommand{\qualitativeseparator}{\par{\color{qualitativerule}\hrule height 0.55pt}}
\newcommand{\qualitativestep}[1]{%
  \par\vspace{2pt}{\centering $\displaystyle #1$\par}\vspace{2pt}%
}
\newcommand{\qualitativeomitted}{\textcolor{qualitativemuted}{[\ldots]}}

\definecolor{stratselect}{HTML}{2CA02C}
\definecolor{strathalve}{HTML}{D62728}
\definecolor{stratfixed}{HTML}{F2A45C}
\definecolor{stratsync}{HTML}{E3C44F}
\definecolor{stratunchanged}{HTML}{8C97A3}
\definecolor{stratavg}{HTML}{7FA9D6}
\newcommand{\stratswatch}[1]{\textcolor{#1}{\rule[-0.1ex]{0.75em}{0.75em}}}

\title{Surprising Success, Repeated Failure:\\Entropy-Guided Credit Assignment for\\Exploration in LLM Reasoning}

\author{Woongyeong Yeo\textsuperscript{1},
    Minki Kang\textsuperscript{1},
    Chanuk Lee\textsuperscript{1},
    Sangwoo Park\textsuperscript{1},
    Jinheon Baek\textsuperscript{1$\dagger$},
    Sung Ju Hwang\textsuperscript{1,2$\dagger$}
    \\
    \textsuperscript{1}KAIST\quad\textsuperscript{2}DeepAuto.ai
    \quad{\footnotesize(\textsuperscript{$\dagger$}: Corresponding authors)}\\
    \faEnvelope[regular]~\texttt{\{wgcyeo, jinheon.baek, sungju.hwang\}@kaist.ac.kr}\\
    \faGlobe~\textbf{Project Page:} \url{https://eapo-explore.github.io}
}

\input{sections/0_abstract}

\begin{document}

\maketitle

\input{sections/1_introduction}
\input{sections/2_preliminaries}
\input{sections/3_motivation}

\input{sections/4_method}
\input{sections/5_experiments}
\input{sections/6_related_work}
\input{sections/7_conclusion}

\bibliography{main}
\bibliographystyle{main}

\clearpage
\appendix
\makeatletter
\setcounter{section}{100}
\renewcommand{\thesection}{\@Alph{\numexpr\c@section-100\relax}}
\renewcommand{\theHsection}{\thesection}
\makeatother
\input{sections/appendix/1_extended_related_work}
\input{sections/appendix/2_add_exp_details}

\input{sections/appendix/3_add_exp_results}

\input{sections/appendix/4_theoretical_analysis}
\input{sections/appendix/5_qualitative_example}

\input{sections/appendix/6_strategy_case}

\end{document}

%% file: math_commands.tex
\usepackage{amsmath,amsfonts,amsthm,bm}

\theoremstyle{plain}
\newtheorem{proposition}{Proposition}

\newcommand{\figleft}{{\em (Left)}}
\newcommand{\figcenter}{{\em (Center)}}
\newcommand{\figright}{{\em (Right)}}
\newcommand{\figtop}{{\em (Top)}}
\newcommand{\figbottom}{{\em (Bottom)}}
\newcommand{\captiona}{{\em (a)}}
\newcommand{\captionb}{{\em (b)}}

\def\eqref#1{equation~\ref{#1}}

\def\1{\bm{1}}

\def\ve{{\bm{e}}}

\def\vp{{\bm{p}}}

\def\vz{{\bm{z}}}

\DeclareMathAlphabet{\mathsfit}{\encodingdefault}{\sfdefault}{m}{sl}
\SetMathAlphabet{\mathsfit}{bold}{\encodingdefault}{\sfdefault}{bx}{n}

\newcommand{\E}{\mathbb{E}}

\newcommand{\softmax}{\mathrm{softmax}}

\newcommand{\KL}{D_{\mathrm{KL}}}
\newcommand{\Var}{\mathrm{Var}}

\newcommand{\Cov}{\mathrm{Cov}}

\DeclareMathOperator{\sign}{sign}

%% file: sections/0_abstract.tex
\begin{abstract}

Reinforcement learning with verifiable rewards (RLVR) enhances reasoning in large language models (LLMs) through outcome-level feedback, yet recent approaches to finer-grained credit assignment often require auxiliary models, additional sampling, or privileged information. Although policy entropy provides a readily available signal, prioritizing uncertain positions under both reinforcement and penalization concentrates penalties where failed responses still retain alternatives for recovery, which can suppress opportunities for exploration. To address this, we introduce \emph{Entropic Advantage Policy Optimization} (\textbf{\EAPO{}}), an entropy-guided credit assignment method that treats success and failure asymmetrically. Specifically, motivated by the observation that success under uncertainty is less repeatable while confident failures tend to recur, \EAPO{} couples normalized policy entropy with the sign of the response advantage to reinforce \surprisingsuccess{} and correct \repeatedfailure{}. It assigns stronger reinforcement to high-entropy decisions in successful responses and stronger penalties to low-entropy decisions in failed responses, while attenuating penalties at uncertain positions to preserve opportunities for recovery. By redistributing the response advantage across tokens, \EAPO{} derives token-level credit directly from existing rollout signals without additional supervision. We validate \EAPO{} on a range of reasoning tasks across both base and reasoning backbones, demonstrating that it achieves the best overall performance. We further show that \EAPO{} promotes more effective exploration, broadening problem coverage and diversifying successful reasoning strategies.

\end{abstract}

%% file: sections/1_introduction.tex
\input{figures/concept}

\section{Introduction}
\label{sec:introduction}

Reinforcement learning with verifiable rewards (RLVR)~\citep{rlvr-tulu3, grpo, deepseek-r1} improves the reasoning capabilities of large language models (LLMs) through outcome-level supervision. The resulting reward evaluates a trajectory as a whole, although the decisions within it may contribute differently to the final outcome. This motivates finer-grained credit assignment that allocates learning feedback to individual reasoning decisions. However, existing approaches to obtaining such feedback require auxiliary models, additional sampling, or access to privileged information~\citep{prime, vineppo, rlrt}.

These requirements motivate the use of \emph{policy entropy}, a measure of next-token uncertainty directly available from the rollout policy, as a signal for token-level feedback allocation. High entropy reflects competing continuations and potential branching points for exploration, whereas low entropy indicates concentrated preferences~\citep{forking-tokens-8020, entropyadv}. Beyond entropy regularization for exploration~\citep{a3c-entropy-bonus}, existing methods select high-entropy tokens for policy updates~\citep{forking-tokens-8020} or reshape token-level advantages~\citep{entropyadv, hapo}. However, these methods incorporate uncertainty in an outcome-agnostic manner, either favoring high-entropy positions under both positive and negative feedback~\citep{forking-tokens-8020, hapo} or adding nonnegative entropy bonuses regardless of outcome~\citep{entropyadv}. The former places the strongest penalties on uncertain positions in failed responses, where competing continuations may still lead to recovery through alternative reasoning paths, potentially constraining further exploration, while the latter focuses only on relieving these penalties and overlooks confident decisions in failed responses, allowing such failures to recur. This raises a central question: \emph{should uncertainty guide both reinforcement and penalization in the same way?}

We examine this question through the asymmetry between reinforcing success and penalizing failure. In a successful trajectory, a high-entropy decision selects among competing continuations, so the observed success may reflect an exploratory path that has not yet become a stable behavioral preference. Stronger reinforcement at these points can help consolidate successful exploration into more repeatable behavior. In contrast, low-entropy decisions within an unsuccessful trajectory can reflect concentrated preferences that favor similar failing continuations, motivating stronger correction of confident patterns associated with failure to discourage their recurrence. Meanwhile, in the same unsuccessful trajectory, high-entropy positions retain competing alternatives, so concentrating penalties on these positions risks constraining opportunities for exploration and recovery. These two observations motivate reinforcing \surprisingsuccess{} and correcting \repeatedfailure{} while preserving uncertain alternatives, without treating entropy as a measure of individual-token correctness.

To incorporate this asymmetry into token-level credit assignment, we introduce \textbf{E}ntropic \textbf{A}dvantage \textbf{P}olicy \textbf{O}ptimization (\textbf{\EAPO{}}), which couples policy entropy with the correctness feedback (i.e., the advantage sign) to reverse the entropy preference between reinforcement and penalization (\Cref{fig:concept}). Specifically, \EAPO{} reweights the response advantage across tokens, assigning stronger reinforcement to high-entropy decisions when the advantage is positive and stronger penalties to low-entropy decisions when it is negative. By reinforcing surprising success and correcting repeated failure in this way, \EAPO{} makes successful exploration more repeatable and prevents confident failures from recurring, while attenuating penalties at uncertain positions within unsuccessful responses to preserve alternative reasoning paths.

We validate \EAPO{} on a range of tasks, including mathematical, logical, and algorithmic reasoning, across both base and reasoning backbones. \EAPO{} achieves the best overall performance, surpassing existing methods for promoting exploration in both average accuracy and problem coverage. Our analyses show that \EAPO{} broadens problem coverage under varying test-time sampling budgets and uncovers a broader set of reasoning strategies, while also confirming the benefits of reinforcing high-entropy decisions and penalizing low-entropy decisions. Together, these findings highlight the benefits of treating uncertainty differently under success and failure to promote more effective exploration in LLM reasoning while preserving alternative reasoning paths.

%% file: figures/concept.tex
\begin{figure}[H]
    \vspace{0.05in}
    \centering
    \includegraphics[width=\linewidth]{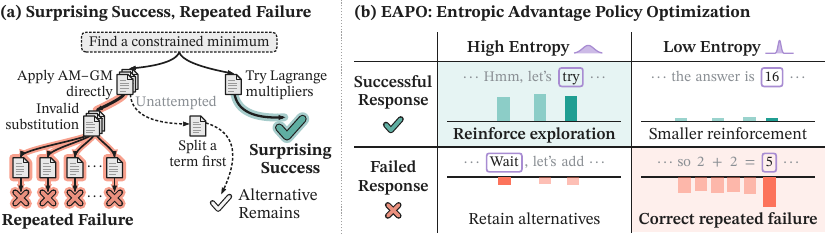}
    \vspace{-0.25in}
    \caption{\textbf{Conceptual illustration.} \captiona{} Success under uncertainty is rarely sampled and thus less repeatable, whereas confident failures keep recurring across many rollouts. \captionb{} \EAPO{} reweights token-level advantages by entropy to reinforce \surprisingsuccess{} at high-entropy positions and correct \repeatedfailure{} at low-entropy positions, while preserving uncertain alternatives.}
    \label{fig:concept}
    \vspace{-0.05in}
\end{figure}

%% file: sections/2_preliminaries.tex
\section{Preliminaries}
\label{sec:preliminaries}

\paragraph{Group Relative Policy Optimization.}
Reinforcement learning with verifiable rewards (RLVR) trains a policy $\pi_\theta$ using rewards computed by task-specific verifiers~\citep{rlvr-tulu3, deepseek-r1}. As one instantiation, group relative policy optimization (GRPO)~\citep{grpo} samples a group of $G>1$ responses $\{y_1,\ldots,y_G\}\sim\pi_\theta(\,\cdot\mid x)$ for each prompt $x$ and normalizes the corresponding response-level rewards $r^i=r(x,y_i)$ to obtain the response advantage
\begin{equation}
    \hat{A}^i
    = \frac{r^i-\operatorname{mean}(\{r^j\}_{j=1}^{G})}
        {\operatorname{std}(\{r^j\}_{j=1}^{G})+\varepsilon}.
    \label{eq:group_advantage}
\end{equation}
The resulting advantage is shared uniformly across all tokens within the response, $\hat{A}_t^i=\hat{A}^i$, and therefore provides no position-specific weighting of the learning signal, treating all token decisions within a response as equally responsible for the final reward.

\paragraph{Policy Entropy.}
For a generation prefix $(x,y_{i,<t})$ and vocabulary $\mathcal{V}$, we define the \emph{policy entropy} as the Shannon entropy~\citep{shannon-entropy} of the policy's next-token distribution:
\begin{equation}
    H_{i,t}
    = -\sum_{v\in\mathcal{V}}
        \pi_\theta(v\mid x,y_{i,<t})
        \log\pi_\theta(v\mid x,y_{i,<t}).
    \label{eq:token_entropy}
\end{equation}
Unlike sampled-token surprisal, $-\log\pi_\theta(y_{i,t}\mid x,y_{i,<t})$, which measures the unexpectedness of a single realized token, $H_{i,t}$ characterizes uncertainty over the full next-token distribution. Higher entropy indicates greater uncertainty among possible continuations, whereas lower entropy reflects more concentrated preferences, with high-entropy positions often associated with branching and exploratory behavior in reasoning~\citep{forking-tokens-8020, entropyadv}.

%% file: sections/3_motivation.tex
\section{Surprising Success, Repeated Failure}
\label{sec:motivation}

Then, how does policy entropy relate to reasoning success? While prior work has highlighted high-entropy positions as branching points for exploring alternatives~\citep{forking-tokens-8020, entropyadv}, this link to exploration does not by itself indicate whether success under uncertainty can be reproduced or whether failure can be avoided even when the policy is confident. In this section, we examine the relationship between policy entropy and successful exploration in reasoning.

\paragraph{High-entropy success is less repeatable, while failure retains alternatives.}

To examine how entropy relates to reproducible success and recovery from failure, we analyze Qwen3-4B/8B-Base on 20 moderately difficult mathematical reasoning problems per model. For each problem, we randomly select two correct and two incorrect responses. Within each selected response, we compare contiguous 32-token windows with high and low mean token entropy, matching their relative-position ranges to control for where resampling begins in the reasoning trajectory. From the fixed prefix preceding each window, we sample 64 continuations and measure their overlap within the first 32 generated tokens and their final-answer accuracy. Further details are provided in \Cref{app:resampling_details}.

\input{tables/regen_entropy}

As shown in \Cref{tab:regen_entropy}, originally correct responses are substantially less likely to succeed again when resampled from high- rather than low-entropy windows, which makes success through uncertainty a \surprisingsuccess{}. In contrast, originally incorrect responses tend to remain incorrect when resampled from low-entropy windows, indicating \repeatedfailure{} despite the policy's confidence. Meanwhile, resampling from high-entropy windows of the same incorrect responses yields higher accuracy, suggesting that these responses retain alternative paths to success despite recurrent failure at low-entropy windows. High-entropy windows also yield lower overlap across both outcomes, supporting their role in exploratory branching. We further show in \Cref{app:recovery_low_entropy} that failed responses remain recoverable from low-entropy windows once the policy departs from its confident continuation, and provide a qualitative example illustrating these trends in \Cref{app:qualitative_example_succ_fail}. These findings highlight the need to make surprising success more repeatable and correct repeated failure at low-entropy positions, while preserving alternative paths to recovery at high-entropy positions in unsuccessful responses.

\paragraph{Token patterns reflect successful exploration and confident failure.}

To characterize the reasoning behaviors behind these trends, we analyze 2,376 Qwen3-4B-Base responses to mathematical reasoning problems, comparing normalized word frequencies between the top and bottom 10\% entropy tails of each correct and incorrect response. As shown in \Cref{fig:word_entropy}, high-entropy positions in correct responses are enriched with exploratory expressions (e.g., ``let's'' and ``consider''), whereas conclusion markers (e.g., ``finally'' and ``confirm'') are more prevalent at low-entropy positions in incorrect responses. These patterns mirror the resampling results, suggesting that \surprisingsuccess{} arises from \emph{successful exploration} that has not yet become a stable preference, while \repeatedfailure{} reflects \emph{confident failure}, in which the policy firmly commits to an incorrect conclusion. Meanwhile, exploratory expressions (e.g., ``try'' and ``instead'') also appear at high-entropy positions in incorrect responses, indicating that failed responses still retain alternatives for recovery, consistent with the gains from high-entropy resampling in \Cref{tab:regen_entropy}. Together, these observations motivate asymmetric credit assignment: stronger reinforcement of high-entropy decisions can make successful exploration repeatable, whereas concentrating penalties on low-entropy decisions can correct confident failure while sparing the exploratory alternatives that unsuccessful responses retain.

\input{figures/word_entropy}

%% file: tables/regen_entropy.tex
\begin{wraptable}[13]{r}{0.45\textwidth}
    \centering
    \vspace{-0.18in}
    \caption{Continuation resampling by original response outcome and window entropy. $\Delta\bar{r}$ denotes the reward difference (High $-$ Low).}
    \label{tab:regen_entropy}
    \vspace{-0.1in}
    \small
    \setlength{\tabcolsep}{3pt}
    \begin{tabular*}{\linewidth}{@{\extracolsep{\fill}}lccccr@{}}
        \toprule
        & \multicolumn{2}{c}{\textbf{Overlap}}
        & \multicolumn{2}{c}{\textbf{Reward}} \\
        \cmidrule(lr){2-3}\cmidrule(lr){4-5}
        Source & Low & High & Low & High & $\Delta\bar{r}$~(pp) \\
        \midrule
        \multicolumn{6}{@{}l@{}}{\textbf{Qwen3-4B-Base}} \\
        Correct   & 29.1 & 4.1 & 0.578 & 0.390 & $\textcolor{success}{-18.8}$ \\
        Incorrect & 27.8 & 7.3 & 0.070 & 0.163 & $\textcolor{failure}{+9.3}$ \\
        \midrule
        \multicolumn{6}{@{}l@{}}{\textbf{Qwen3-8B-Base}} \\
        Correct   & 29.6 & 4.3 & 0.625 & 0.418 & $\textcolor{success}{-20.7}$ \\
        Incorrect & 30.2 & 6.9 & 0.032 & 0.159 & $\textcolor{failure}{+12.7}$ \\
        \bottomrule
    \end{tabular*}
    \vspace{-0.1in}
\end{wraptable}

%% file: figures/word_entropy.tex
\begin{figure}[t]
    \centering
    \captionsetup[subfigure]{position=top,font=small,skip=2pt,justification=centering}
    \begin{subfigure}[t]{0.52\linewidth}
        \centering
        \caption{Correct responses}
        \label{fig:word_entropy_correct}
        \includegraphics[width=\linewidth]{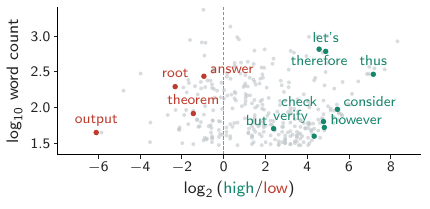}
    \end{subfigure}%
    \begin{subfigure}[t]{0.48\linewidth}
        \centering
        \caption{Incorrect responses}
        \label{fig:word_entropy_incorrect}
        \includegraphics[width=\linewidth]{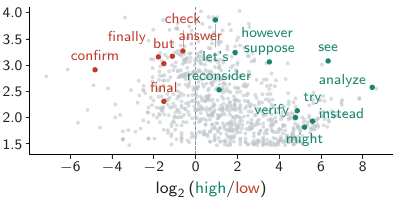}
    \end{subfigure}
    \vspace{-0.15in}
    \caption{\textbf{Word usage by token entropy.} Green and red mark words with higher relative frequencies in the top and bottom 10\% entropy tails of each Qwen3-4B-Base response, respectively. Exploratory expressions are enriched at high-entropy positions in both correct and incorrect responses, whereas conclusion markers are enriched at low-entropy positions in incorrect responses.}
  \label{fig:word_entropy}
\end{figure}

%% file: sections/4_method.tex
\section{Entropic Advantage Policy Optimization}
\label{sec:method}

We present \textbf{E}ntropic \textbf{A}dvantage \textbf{P}olicy \textbf{O}ptimization (\textbf{\EAPO{}}), which operationalizes this asymmetric credit assignment by redistributing the response advantage across completion tokens using only the policy's own entropy, assigning stronger reinforcement to uncertain decisions in positive-advantage rollouts and stronger penalties to confident decisions in negative-advantage rollouts.

\subsection{Asymmetric Advantage Reweighting via Sign--Entropy Coupling}
\label{sec:eapo_entropy_allocation}

An entropy preference shared across advantage signs would prioritize the same positions under reinforcement and penalization, so favoring high entropy would also concentrate penalties on uncertain decisions in negative-advantage responses. To address this, we couple entropy with the advantage sign, reversing the preference between reinforcement and penalization.

\paragraph{Policy Entropy Normalization.}
Absolute token entropies can vary substantially across responses and fluctuate sharply within a response, making absolute entropy-based reweighting sensitive to noise and potentially destabilizing policy optimization. To address this, we normalize token entropies using shared percentiles within each rollout batch $\mathcal{B}$. Specifically, we evaluate the entropy $H_{i,t}$ in \Cref{eq:token_entropy} under the old policy $\pi_{\mathrm{old}}$. For response $i$, let $\mathcal{I}_i$ denote valid completion positions, excluding prompt and padding tokens, and let $T_i=|\mathcal{I}_i|>0$. We define
\begin{equation}
    h_{i,t}
    = \operatorname{stopgrad}\!\left[
        \operatorname{clip}\!\left(
        \frac{H_{i,t}-Q_{\mathrm{lo}}}{Q_{\mathrm{hi}}-Q_{\mathrm{lo}}+\varepsilon_H},0,1
        \right)
    \right],
    \qquad t\in\mathcal{I}_i,
    \label{eq:normalized_entropy}
\end{equation}
where $Q_{\mathrm{lo}}$ and $Q_{\mathrm{hi}}$ are the 10th and 90th percentiles of valid completion-token entropies in $\mathcal{B}$. We use clipping to limit the influence of extreme entropies and $\operatorname{stopgrad}$ to use the entropies as detached credit assignment signals, preventing gradients from propagating through the reweighting factors.

\paragraph{Asymmetric Advantage Redistribution.}
We design token-level credit to depend jointly on the advantage sign and normalized entropy. Given the response advantage $\hat{A}^i$ in \Cref{eq:group_advantage}, we express this coupling through the signed uncertainty $\textcolor{eapopurple}{\sign(\hat{A}^i)h_{i,t}}$. For $t\in\mathcal{I}_i$ and $\kappa\ge0$, we define
\begin{equation}
    w_{i,t}
    = \frac{\exp\!\left(\kappa\,\textcolor{eapopurple}{\sign(\hat{A}^i)h_{i,t}}\right)}
        {\frac{1}{T_i}\sum_{u\in\mathcal{I}_i}
        \exp\!\left(\kappa\,\textcolor{eapopurple}{\sign(\hat{A}^i)h_{i,u}}\right)},
    \qquad \hat{A}_t^{i,\mathrm{E}}=\hat{A}^iw_{i,t}.
    \label{eq:entropic_allocation}
\end{equation}
The exponential weighting in \Cref{eq:entropic_allocation} favors larger signed uncertainty, with $\kappa$ controlling concentration. For $\kappa>0$, this assigns greater weight to high-entropy decisions when $\hat{A}^i>0$ and to low-entropy decisions when $\hat{A}^i<0$, with the ratio between any two weights within a response bounded by $e^\kappa$. Dividing by the within-response mean of the exponential scores ensures that the advantage is redistributed across tokens rather than rescaled for the response as a whole. We then multiply $\hat{A}^i$ by $w_{i,t}$ to obtain the token-level advantage $\hat{A}_t^{i,\mathrm{E}}$. Since the weights are positive, they preserve the sign of $\hat{A}^i$ while determining how strongly each decision is reinforced or penalized.

\paragraph{Policy Optimization.}
\EAPO{} simply replaces the response-level advantage $\hat{A}^i$ with the token-level advantage $\hat{A}_t^{i,\mathrm{E}}$ in the original policy objective, leaving the rest of the optimization procedure unchanged. Since token-level credit is derived directly from the policy's own entropy and the existing response advantage, \EAPO{} requires no auxiliary models, token-level supervision, or substantial additional computation, making it readily applicable to existing policy optimization pipelines.

\subsection{Theoretical Interpretation}
\label{sec:eapo_theory}

We now turn to the theoretical interpretation of \EAPO{}, connecting its asymmetric credit allocation to KL anchoring in token space and examining how penalty attenuation affects conditional entropy and distributional distortion among unchosen alternatives compared with uniform weighting.

\paragraph{Reinterpreting KL Anchoring in Token Space.}
KL-regularized RLVR anchors the policy to a reference policy $\pi_{\mathrm{ref}}$ through $\KL(\pi_\theta\Vert\pi_{\mathrm{ref}})$ while optimizing verifiable rewards~\citep{grpo}. We transfer this principle from policy space to credit allocation over token positions, with uniform credit as the reference and signed uncertainty $\sign(\hat{A}^i)h_{i,t}$ as the utility.

\begin{proposition}[KL-regularized uncertainty allocation]
\label{prop:kl_allocation}
Let $u_{i,t}=1/T_i$ be uniform credit over $\mathcal{I}_i$, and let $\Delta(\mathcal{I}_i)$ denote the probability simplex on these positions. For $\kappa>0$, the problem
\begin{equation}
    \max_{q\in\Delta(\mathcal{I}_i)}
    \left\{
        \sign(\hat{A}^i)\E_{t\sim q}[h_{i,t}]
        -\frac{1}{\kappa}\KL(q\Vert u_i)
    \right\}
    \label{eq:kl_allocation}
\end{equation}
has the unique solution $q^*_{i,t}\propto u_{i,t}\exp\!\left(\kappa\sign(\hat{A}^i)h_{i,t}\right)$.
\end{proposition}

We defer the proof to \Cref{app:kl_allocation_proof}. Our formulation in \Cref{eq:entropic_allocation} is therefore an operationalization of asymmetric credit assignment based on signed uncertainty, obtained by transferring KL anchoring to token space with uniform allocation as the reference. Increasing $\kappa$ weakens this anchor and strengthens the preference for high entropy under positive advantages and low entropy under negative advantages, while $\kappa\to0$ recovers uniform credit.

\paragraph{Understanding the Role of Penalty Attenuation.}
Motivated by the observation that updates with negative advantages can shift probability toward already likely alternatives~\citep{dpo-squeezing}, we analyze how this effect depends on update strength to understand \EAPO{}'s penalty attenuation at uncertain positions. Specifically, at a fixed prefix, we consider a single policy-gradient step on the logits with negative advantage $\hat{A}<0$ and update strength $\beta=\eta|\hat{A}|w\ge0$, where $w$ is the allocation weight and $\eta>0$ is the effective step size. Let $\nu^{(\beta)}$ denote the conditional distribution over unchosen tokens after the update, with $\nu=\nu^{(0)}$ denoting their initial conditional distribution.

\begin{proposition}[Concentration under negative updates]
\label{prop:alternative_concentration}
For a single negative logit update with finite logits and the initial distribution held fixed, the sampled-token probability strictly decreases with $\beta$, while $H(\nu^{(\beta)})$ is nonincreasing and $\KL(\nu^{(\beta)}\Vert\nu)$ is nondecreasing.
\end{proposition}

Stronger penalties suppress the sampled token and favor already likely alternatives, sharpening the distribution over the remaining alternatives. Within this single-step logit model, reducing the penalty at a given position limits this sharpening and KL distortion. The proof and additional analyses are provided in \Cref{app:concentration_proof,app:allocation_implications}. Together with the findings in \Cref{sec:motivation}, this analysis supports attenuating penalties at uncertain positions while concentrating correction on repeated failure.

%% file: sections/5_experiments.tex
\section{Experiments}
\label{sec:experiments}

\input{tables/main_results}

\subsection{Experimental Setup}
\label{sec:experimental_setup}

\paragraph{Benchmarks \& Metrics.}
We evaluate \EAPO{} on six competition-level mathematical reasoning benchmarks: AIME24/25/26~\citep{aime}, HMMT26~\citep{hmmt}, AMC23~\citep{amc}, and the level-5 subset of MATH500~\citep{math500-verify-step-by-step}. We report avg@32, the mean accuracy over 32 sampled responses per problem, and pass@32, the probability of generating at least one correct response within 32 samples, which we estimate using the unbiased estimator of \citet{unbiased-passk}.

\paragraph{Baselines.}
We compare \EAPO{} against the initial backbone and five RLVR methods, including GRPO~\citep{grpo}. EntropyAdv~\citep{entropyadv} adds an entropy term to the advantage to encourage exploration, while HAPO~\citep{hapo} applies a bounded, sign-preserving adjustment based on normalized policy entropy to emphasize high-entropy positions. 80/20 (Forking Tokens)~\citep{forking-tokens-8020} restricts policy-gradient updates to the highest-entropy 20\% of tokens. Finally, RLRT~\citep{rlrt} reverses the self-distillation signal on correct rollouts, reinforcing tokens more probable under the student than under a teacher conditioned on privileged information.

\paragraph{Implementation Details.}
We primarily evaluate \EAPO{} across two model types, with Qwen3-4B-Base and Qwen3-8B-Base~\citep{qwen3} as base backbones and Qwen3-4B~\citep{qwen3} and Olmo-3-7B-Think-DPO~\citep{olmo3} as reasoning backbones. All optimization-based methods are trained on DAPO-Math-17k-Processed with a DAPO-style training configuration~\citep{dapo}. Further experimental and implementation details are provided in \Cref{app:add_exp_details}.

\subsection{Main Results}
\label{sec:main_results}

\Cref{tab:main_results} compares \EAPO{} with baseline methods across six mathematical reasoning benchmarks. \EAPO{} achieves the best overall performance, with the highest mean avg@32 and pass@32 for all four backbones. Specifically, with Qwen3-4B-Base and Qwen3-8B-Base, \EAPO{} attains mean accuracies of 31.0\% and 34.0\%, surpassing the strongest entropy-based baselines (EntropyAdv and 80/20, respectively) by substantial margins of 5.6 and 4.3 percentage points. Notably, \EAPO{} also outperforms the strongest baseline, RLRT, in both metrics on both base backbones even without relying on privileged information. Meanwhile, the gains extend to reasoning backbones, where \EAPO{} achieves mean accuracies of 72.4\% with Qwen3-4B and 74.3\% with Olmo-3-7B-Think-DPO, surpassing the strongest baseline in each case despite their substantially stronger initial performance. These results demonstrate consistent gains across benchmarks and model types.

\input{tables/ood_results}
\input{figures/training_dynamics}

\paragraph{OOD Generalization.}
To examine whether these improvements extend beyond mathematical reasoning, we evaluate the two base backbones on eight tasks from Reasoning Gym~\citep{reasoning-gym}, spanning logical, spatial, and algorithmic reasoning (with details of each task in \Cref{app:add_exp_details}). As shown in \Cref{tab:ood_results}, \EAPO{} achieves the highest macro-averaged scores of 34.89\% and 43.02\% with Qwen3-4B-Base and Qwen3-8B-Base, improving upon the respective best baselines, 80/20 and HAPO, by 1.46 and 3.02 percentage points. These results suggest that the benefits of \EAPO{} generalize to diverse reasoning tasks beyond the mathematical problems used for the optimization.

\subsection{Training Dynamics}
\label{sec:training_dynamics}

To understand how \EAPO{} behaves during training, we examine the evolution of training reward, response length, and average entropy. \Cref{fig:training_dynamics} shows that \EAPO{} maintains higher training rewards than the baselines over most of training, with its advantage becoming particularly pronounced near the end, further supporting its superiority over the baselines even during training. Meanwhile, entropy-based RLVR approaches consistently exhibit an overall increase in response length, producing longer responses than GRPO. \EAPO{} also follows this trend, sustaining longer reasoning trajectories as training progresses. However, longer responses alone do not establish more effective exploration. To examine whether response length explains the performance gains, we extend baseline responses following \citet{s1} and compare accuracy at comparable mean response lengths in \Cref{app:token_budget}, where additional continuations do not close the gap to \EAPO{}.

The average entropy exhibits an interesting trend throughout training. While retaining higher entropy than GRPO during later training, \EAPO{} ends with lower average entropy than EntropyAdv and HAPO. A key distinction is that \EAPO{} redistributes the original response-level advantage across tokens while preserving its within-response mean, whereas these baselines apply entropy-dependent additive adjustments. In particular, EntropyAdv adds a nonnegative entropy term to the advantage to encourage high-entropy actions and exhibits the highest final entropy in our experiments. \EAPO{}, however, achieves stronger reasoning performance without a substantial increase in entropy, suggesting that higher average entropy is not a prerequisite for more effective exploration.

\subsection{Analysis of Exploration}
\label{sec:exploration_analysis}

To assess whether \EAPO{}'s asymmetric design broadens exploration, we compare pass@$k$, strategy coverage, response diversity, and epistemic marker frequency across different exploration methods.

\input{figures/exploration}

\paragraph{Exploration across Sampling Budgets.}
We evaluate pass@$k$ on AIME26 with Qwen3-4B-Base, varying the sampling budget from 1 to 256. As shown in \Cref{fig:exploration} (Left), entropy-based methods, EntropyAdv and HAPO, expand coverage over GRPO, but the gains remain modest. RLRT, which uses privileged information to guide exploration, also yields only a limited improvement over GRPO. In contrast, \EAPO{} consistently achieves the highest pass@$k$ across all sampling budgets, maintaining a substantial margin over GRPO and the other baselines. In particular, at $k=256$, \EAPO{} solves an additional problem that none of the baselines solve, suggesting that its exploration broadens problem coverage even after extensive sampling. Results on other models and benchmarks, which show similar trends of broader coverage across all sampling budgets, are provided in \Cref{app:passk_other_benchmarks}.

\paragraph{Discovering Distinct Successful Strategies.}
Pass@$k$ alone does not reveal whether broader coverage stems from distinct successful reasoning strategies or from repeatedly reaching correct answers via the same strategy. To examine this, we label each response's strategy with a two-stage LLM judge: GPT-6-Luna~\citep{gpt6-sol-luna} first summarizes each response's final approach, and GPT-6.1-Sol~\citep{gpt6-1-sol} then clusters each problem's shuffled, anonymized summaries into strategies, hiding method identity and correctness. We pool the strategies of correct responses from all methods into a per-problem reference set $S^*_p$ and define strategy coverage@$k$ of method $m$ as
\begin{equation}
    \mathrm{Coverage@}k = \frac{1}{|P|}\sum_{p\in P}\mathbb{E}\left[\frac{|S_{p,m}(k)|}{|S^*_p|}\right],
    \label{eq:strategy_coverage}
\end{equation}
where $S_{p,m}(k)$ denotes the successful strategies of method $m$ within $k$ samples and $P$ is the set of problems solved by any method. As shown in \Cref{fig:exploration} (Center), \EAPO{} recovers the largest fraction of successful strategies at every budget, reaching 62.2\% at $k=32$, 10.6 points above GRPO and 3.3 points above the strongest baseline, RLRT. This indicates that \EAPO{} does not merely solve more problems but discovers more distinct successful reasoning strategies. Specifically, as shown in \Cref{tab:strategy_recurrence}, \EAPO{} suppresses the repeated failure of the initial policy, where confident but incorrect strategies persist under the other methods, while preserving alternatives, thereby discovering successful strategies that the initial policy and the other methods rarely attempt, as illustrated in \Cref{app:strategy_case}. We provide further details and results in \Cref{app:strategy_coverage_details,app:strategy_within_problem}.

\paragraph{Response Diversity.}
However, even within the same strategy, the model can fall into repeated failure, converging to the same incorrect answer. We thus examine final-answer diversity on challenging problems, where such failures are most prevalent, which can be measured exactly without an LLM judge. Specifically, we consider all problems from the main Qwen3-4B-Base experiments, using $N=32$ sampled responses per problem and selecting those with avg@32 $\leq 25\%$ under the initial model. We measure normalized answer entropy ($H_{\mathrm{norm}}$) and the collision rate, the fraction of distinct response pairs sharing the same final answer:
\begin{equation}
    C = \frac{\sum_i n_i(n_i-1)}{N(N-1)},
    \label{eq:answer_collision}
\end{equation}
where $n_i$ counts responses with final answer $i$, with $\sum_i n_i=N$. \Cref{tab:answer_diversity} shows that the other RLVR methods exhibit final-answer diversity comparable to GRPO, whereas \EAPO{} achieves higher normalized answer entropy and a lower collision rate. This suggests that \EAPO{} maintains broader exploration on challenging problems, producing a more diverse range of candidate solutions even when they are incorrect. Such diversity offers more opportunities to discover correct solutions, consistent with the broader coverage in the pass@$k$ and strategy coverage analyses above.

\paragraph{Exploratory Cues.}
Beyond answer diversity, we further examine exploratory behavior within the generated responses via \emph{epistemic markers}, tokens that signal exploratory behavior during reasoning (e.g., ``wait''), following \citet{epistemic-marker}. \Cref{fig:exploration} (Right) shows that, while other exploration-oriented methods tend to generate epistemic tokens more frequently than GRPO, \EAPO{} exhibits the highest frequency per 1,000 generated tokens across all four mathematical benchmarks by a substantial margin. This pattern suggests that \EAPO{}'s combination of exploration promotion and penalty attenuation encourages the generation of such cues, supporting continued exploration of alternative reasoning paths.

\subsection{Ablation Studies}
\label{sec:ablation_studies}

\input{tables/diversity_coupling}

\paragraph{Effect of Sign--Entropy Coupling.}
To isolate how entropy preferences under reinforcement and penalization affect performance, we independently vary the allocation direction for positive- and negative-advantage responses. Specifically, we replace $\sign(\hat{A}^i)$ in \Cref{eq:entropic_allocation} with $b_{+}$ when $\hat{A}^i>0$ and $b_{-}$ when $\hat{A}^i<0$. Each coefficient is chosen from $\{-1,0,+1\}$, corresponding to low-entropy preference, uniform credit, or high-entropy preference, respectively. Thus, $(b_{+},b_{-})=(+1,-1)$ recovers \EAPO{}, whereas $(0,0)$ yields uniform token credit. We evaluate all nine combinations on Qwen3-4B-Base, reporting macro-averaged scores across the six benchmarks. \Cref{tab:sign_entropy_coupling} shows that performance improves as the entropy preference shifts from low to high for positive advantages and from high to low for negative advantages. Specifically, fixing high-entropy reinforcement ($b_{+}=+1$) and shifting penalization from high to low entropy improves accuracy by 5.64 percentage points. Similarly, fixing low-entropy penalization ($b_{-}=-1$) and shifting reinforcement from low to high entropy improves accuracy by 7.01 percentage points, indicating that both directions contribute to the gains of \EAPO{}. Consequently, $(b_{+},b_{-})=(+1,-1)$, the allocation used by \EAPO{}, achieves the highest avg@32 and pass@32, improving avg@32 over uniform credit by 6.52 percentage points, which supports \EAPO{}'s asymmetric allocation directions for reinforcement and penalization. We provide extended results with Qwen3-8B-Base, showing similar trends, in \Cref{app:ext_se_couple}.

\input{tables/kappa_sensitivity}\paragraph{Sensitivity to $\kappa$.}
\leavevmode\EAPO{} introduces a single hyperparameter, $\kappa$, to control the concentration of entropy-guided token credit, with $\kappa=\log K$ allowing token weights within a response to differ by up to a factor of $K$. During training, we observe that larger $\kappa$ generally leads to faster improvements but also greater instability, indicating a trade-off between credit concentration and optimization stability. We select $\kappa=\log 4$ as the default based on the observed training stability. \Cref{tab:kappa_sensitivity} shows that increasing $\kappa$ generally improves overall performance. As $\kappa$ approaches zero, the allocation converges to uniform token credit (i.e., GRPO), and the accompanying performance drop supports using entropy to distinguish token contributions to reinforcement and penalization.

%% file: tables/main_results.tex
\begin{table}[t]
    \centering
    \caption{Results of diverse RLVR methods with base and reasoning models across six mathematical reasoning benchmarks. \textbf{Bold} denotes the best performance within each backbone.}
    \label{tab:main_results}
    \vspace{-0.1in}
    \small
    \setlength{\tabcolsep}{2.5pt}
    \resizebox{\linewidth}{!}{%
    \begin{tabular}{p{3.1cm}*{12}{>{\centering\arraybackslash}p{0.9cm}}}
        \toprule
        \multirow{2}{*}{\textbf{Method}}
        & \multicolumn{2}{c}{\textbf{AIME24}}
        & \multicolumn{2}{c}{\textbf{AIME25}}
        & \multicolumn{2}{c}{\textbf{AIME26}}
        & \multicolumn{2}{c}{\textbf{HMMT26}}
        & \multicolumn{2}{c}{\textbf{AMC23}}
        & \multicolumn{2}{c}{\textbf{MATH500-H}} \\
        \cmidrule(lr){2-3}\cmidrule(lr){4-5}\cmidrule(lr){6-7}
        \cmidrule(lr){8-9}\cmidrule(lr){10-11}\cmidrule(lr){12-13}
        & {\scriptsize Avg@32} & {\scriptsize Pass@32}
        & {\scriptsize Avg@32} & {\scriptsize Pass@32}
        & {\scriptsize Avg@32} & {\scriptsize Pass@32}
        & {\scriptsize Avg@32} & {\scriptsize Pass@32}
        & {\scriptsize Avg@32} & {\scriptsize Pass@32}
        & {\scriptsize Avg@32} & {\scriptsize Pass@32} \\
        \midrule
        \textit{Qwen3-4B-Base}
                   & 10.7 & 43.3 & 5.6 & 30.0 & 6.8 & 23.3 & 3.9 & 24.2 & 42.3 & 90.0 & 43.0 & 82.8 \\
        \quad GRPO       & 12.5 & 33.3 & 10.4 & 36.7 & 7.5 & 33.3 & 5.9 & 24.2 & 54.5 & 90.0 & 56.1 & 85.1 \\
        \quad EntropyAdv & 13.1 & 46.7 & 12.5 & 36.7 & 10.0 & 33.3 & 6.3 & 27.3 & 53.8 & 90.0 & 56.7 & 86.6 \\
        \quad HAPO       & 12.2 & 36.7 & 11.0 & 40.0 & 8.9 & 30.0 & 6.5 & 24.2 & 52.0 & 87.5 & 53.2 & 85.1 \\
        \quad 80/20      & 11.9 & 33.3 & 12.7 & 40.0 & 8.6 & 30.0 & 6.6 & 24.2 & 54.1 & \textbf{95.0} & 52.8 & 84.3 \\
        \quad RLRT       & 15.3 & 50.0 & 12.7 & 33.3 & 14.2 & 40.0 & 8.9 & \textbf{36.4} & 54.9 & \textbf{95.0} & 57.4 & 86.6 \\
        \rowcolor{ours!15}
        \quad \textbf{\EAPO{} (Ours)}
                   & \textbf{20.2} & \textbf{53.3} & \textbf{17.1} & \textbf{46.7} & \textbf{15.7} & \textbf{50.0} & \textbf{13.4} & 33.3 & \textbf{56.3} & 92.5 & \textbf{63.3} & \textbf{91.8} \\
        \midrule

        \textit{Qwen3-8B-Base}
                   & 12.8 & 46.7 & 13.4 & 43.3 & 9.0 & 33.3 & 5.8 & 27.3 & 54.0 & 92.5 & 47.7 & 85.1 \\
        \quad GRPO       & 15.4 & 40.0 & 13.4 & 40.0 & 10.5 & 33.3 & 5.5 & 24.2 & 57.3 & 87.5 & 58.6 & 88.1 \\
        \quad EntropyAdv & 15.9 & 33.3 & 15.2 & 43.3 & 9.9 & 40.0 & 6.0 & 27.3 & 57.8 & 87.5 & 59.2 & 85.8 \\
        \quad HAPO       & 15.6 & 40.0 & 12.9 & 26.7 & 8.4 & 36.7 & 5.1 & 27.3 & 61.5 & 87.5 & 57.4 & 87.3 \\
        \quad 80/20      & 17.2 & 46.7 & 15.1 & 33.3 & 13.9 & 40.0 & 8.5 & 27.3 & 60.6 & 92.5 & 63.1 & 87.3 \\
        \quad RLRT       & 21.7 & 56.7 & \textbf{20.1} & \textbf{46.7} & 17.3 & 56.7 & 10.6 & \textbf{30.3} & 64.2 & 95.0 & 64.6 & 91.0 \\
        \rowcolor{ours!15}
        \quad \textbf{\EAPO{} (Ours)}
                   & \textbf{24.1} & \textbf{63.3} & 19.3 & \textbf{46.7} & \textbf{17.9} & \textbf{60.0} & \textbf{11.2} & \textbf{30.3} & \textbf{65.7} & \textbf{97.5} & \textbf{65.9} & \textbf{95.5} \\
        \midrule

        \textit{Qwen3-4B}
                   & 67.9 & 83.3 & 61.7 & 86.7 & 60.6 & \textbf{86.7} & 40.1 & 66.7 & 95.2 & \textbf{100.0} & 90.0 & 95.5 \\
        \quad GRPO       & 70.6 & 83.3 & 62.8 & 86.7 & 63.2 & 83.3 & 40.9 & 69.7 & 94.9 & \textbf{100.0} & 89.8 & 96.3 \\
        \quad EntropyAdv & 73.0 & \textbf{86.7} & 62.2 & 86.7 & 64.0 & \textbf{86.7} & 42.2 & 69.7 & 95.9 & \textbf{100.0} & \textbf{90.6} & 96.3 \\
        \quad HAPO       & 72.0 & \textbf{86.7} & 64.1 & 83.3 & 62.0 & 83.3 & 41.6 & 66.7 & 94.9 & \textbf{100.0} & 89.9 & 96.3 \\
        \quad 80/20      & 71.4 & \textbf{86.7} & 63.2 & 83.3 & 61.0 & \textbf{86.7} & 40.8 & 66.7 & 94.3 & \textbf{100.0} & 90.5 & 96.3 \\
        \quad RLRT       & 72.5 & \textbf{86.7} & 63.2 & 86.7 & 62.7 & 83.3 & 42.8 & 69.7 & 95.9 & \textbf{100.0} & \textbf{90.6} & 96.3 \\
        \rowcolor{ours!15}
        \quad \textbf{\EAPO{} (Ours)}
                   & \textbf{73.5} & \textbf{86.7} & \textbf{65.0} & \textbf{90.0} & \textbf{64.8} & \textbf{86.7} & \textbf{43.2} & \textbf{72.7} & \textbf{97.7} & \textbf{100.0} & 90.2 & \textbf{97.0} \\
        \midrule

        \textit{Olmo-3-7B-Think-DPO}
                   & 74.1 & 90.0 & 60.1 & 80.0 & 66.3 & \textbf{90.0} & 42.7 & 72.7 & 95.5 & \textbf{100.0} & 90.1 & \textbf{96.3} \\
        \quad GRPO       & 75.9 & \textbf{93.3} & 62.2 & 83.3 & 68.4 & \textbf{90.0} & \textbf{46.0} & 69.7 & 95.8 & \textbf{100.0} & 91.4 & \textbf{96.3} \\
        \quad EntropyAdv & 74.1 & \textbf{93.3} & 64.3 & 83.3 & 69.4 & \textbf{90.0} & 45.7 & 72.7 & 94.6 & \textbf{100.0} & 90.4 & \textbf{96.3} \\
        \quad HAPO       & 74.2 & \textbf{93.3} & 63.3 & 86.7 & 66.4 & \textbf{90.0} & 43.8 & 69.7 & 95.7 & \textbf{100.0} & 90.6 & \textbf{96.3} \\
        \quad 80/20      & 75.6 & \textbf{93.3} & 65.1 & 83.3 & 69.4 & \textbf{90.0} & 44.9 & 75.8 & 96.0 & \textbf{100.0} & 91.3 & \textbf{96.3} \\
        \quad RLRT       & \textbf{77.3} & \textbf{93.3} & 63.4 & 83.3 & 68.1 & 86.7 & 44.6 & 69.7 & 95.2 & \textbf{100.0} & 91.3 & \textbf{96.3} \\
        \rowcolor{ours!15}
        \quad \textbf{\EAPO{} (Ours)}
                   & 76.7 & \textbf{93.3} & \textbf{65.3} & \textbf{90.0} & \textbf{70.3} & \textbf{90.0} & 45.5 & \textbf{78.8} & \textbf{96.2} & \textbf{100.0} & \textbf{91.8} & \textbf{96.3} \\
        \bottomrule
    \end{tabular}%
    }
\end{table}

%% file: tables/ood_results.tex
\begin{table}[t]
    \centering
    \caption{OOD generalization on Reasoning Gym tasks (avg@16). \textbf{Bold} denotes best performance.}
    \label{tab:ood_results}
    \vspace{-0.1in}
    \small
    \setlength{\tabcolsep}{4pt}
    \resizebox{\linewidth}{!}{%
    \begin{tabular}{p{3.1cm}*{9}{>{\centering\arraybackslash}p{1.05cm}}}
        \toprule
        \textbf{Method} & \textbf{Graph} & \textbf{Cube} & \textbf{Sudoku} & \textbf{Family} & \textbf{K\&K} & \textbf{Anag.} & \textbf{Zebra} & \textbf{Palin.} & \textbf{Avg.} \\
        \midrule
        \textit{Qwen3-4B-Base} & 6.75 & 24.63 & 40.00 & 17.00 & 29.13 & 37.50 & 29.50 & 9.00 & 24.19 \\
        \quad GRPO & 7.50 & 28.13 & 41.88 & 25.38 & 38.63 & 43.50 & 41.88 & 11.38 & 29.78 \\
        \quad EntropyAdv & 8.38 & 27.00 & 37.00 & 27.00 & 41.00 & 45.38 & 41.88 & 14.50 & 30.27 \\
        \quad HAPO & 10.13 & 27.13 & 40.38 & 30.13 & 43.63 & \textbf{47.88} & 47.50 & 14.13 & 32.61 \\
        \quad 80/20 & 11.38 & 26.75 & 39.25 & \textbf{31.50} & 49.63 & 46.38 & 47.38 & 15.13 & 33.43 \\
        \quad RLRT & 7.38 & 26.75 & 38.50 & 26.50 & 43.38 & 42.88 & 44.13 & 13.63 & 30.39 \\
        \rowcolor{ours!15}
        \quad \textbf{\EAPO{} (Ours)} & \textbf{12.88} & \textbf{28.25} & \textbf{42.13} & 30.88 & \textbf{52.88} & 46.00 & \textbf{48.38} & \textbf{17.75} & \textbf{34.89} \\
        \midrule
        \textit{Qwen3-8B-Base} & 12.25 & 22.88 & 39.38 & 38.38 & 33.88 & 38.63 & 44.88 & 9.00 & 29.91 \\
        \quad GRPO & 18.50 & 25.63 & 40.25 & 42.50 & 56.88 & 54.00 & 49.50 & 19.13 & 38.30 \\
        \quad EntropyAdv & 16.75 & 22.25 & 47.88 & 44.25 & 49.38 & 55.63 & 54.25 & 18.13 & 38.56 \\
        \quad HAPO & 20.13 & 24.63 & 40.88 & 45.63 & 59.13 & 56.50 & 54.25 & 18.88 & 40.00 \\
        \quad 80/20 & 17.75 & 23.13 & 37.38 & 44.75 & 60.13 & 52.63 & 51.00 & 18.88 & 38.21 \\
        \quad RLRT & 19.00 & 22.38 & 46.25 & 40.75 & 48.50 & 43.88 & 51.75 & 18.75 & 36.41 \\
        \rowcolor{ours!15}
        \quad \textbf{\EAPO{} (Ours)} & \textbf{20.50} & \textbf{26.25} & \textbf{52.25} & \textbf{46.25} & \textbf{60.50} & \textbf{62.63} & \textbf{54.50} & \textbf{21.25} & \textbf{43.02} \\
        \bottomrule
    \end{tabular}%
    }
\end{table}

%% file: figures/training_dynamics.tex
\begin{figure}[t!]
    \centering
    \includegraphics[width=\linewidth]{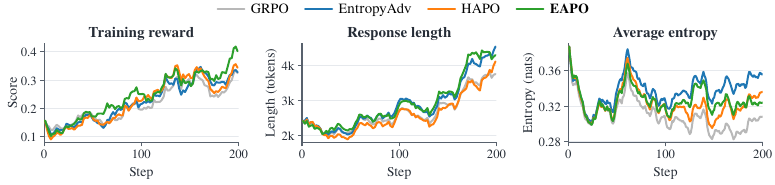}
    \vspace{-0.25in}
    \caption{\textbf{Training dynamics.} \figleft{} Training reward on Qwen3-4B-Base; \figcenter{} response length and \figright{} average entropy on Qwen3-4B, using a trailing 15-step moving average.}
    \label{fig:training_dynamics}
\end{figure}

%% file: figures/exploration.tex
\begin{figure}[t!]
    \centering
    \includegraphics[width=\linewidth]{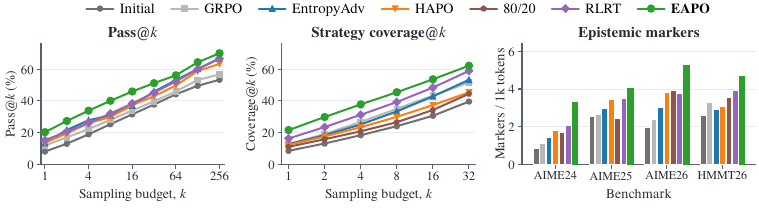}
    \vspace{-0.15in}
    \caption{\textbf{Analysis of exploration} with Qwen3-4B-Base. \figleft{} Pass@$k$ on AIME26. \figcenter{} Strategy coverage@$k$, the fraction of distinct successful strategies recovered within $k$ samples, and \figright{} epistemic markers per 1,000 generated tokens on AIME24--26 and HMMT26.}
    \label{fig:exploration}
\end{figure}

%% file: tables/diversity_coupling.tex
\begin{table}[t!]
    \centering
    \begin{minipage}[t]{0.33\linewidth}
        \caption{\textbf{Answer diversity.} Answer entropy and collision rate on low-accuracy problems.}
        \label{tab:answer_diversity}
        \vspace{-0.1in}
        \footnotesize
        \setlength{\tabcolsep}{1pt}
        \renewcommand{\arraystretch}{0.89}
        \begin{tabular}{p{\dimexpr0.37\linewidth-2\tabcolsep\relax}>{\centering\arraybackslash}p{\dimexpr0.30\linewidth-2\tabcolsep\relax}>{\centering\arraybackslash}p{\dimexpr0.33\linewidth-2\tabcolsep\relax}}
            \toprule
            \textbf{Method} & $H_{\mathrm{norm}}$\,($\uparrow$) & \textbf{Collision}\,($\downarrow$) \\
            \midrule
            GRPO       & 0.657 & 0.166 \\
            EntropyAdv & 0.659 & 0.159 \\
            HAPO       & 0.647 & 0.173 \\
            80/20      & 0.672 & 0.161 \\
            RLRT       & 0.659 & 0.168 \\
            \rowcolor{ours!15}
            \textbf{\EAPO{}} & \textbf{0.751} & \textbf{0.101} \\
            \bottomrule
        \end{tabular}
    \end{minipage}\hfill
    \begin{minipage}[t]{0.64\linewidth}
        \caption{\textbf{Ablation on sign--entropy coupling.} Macro average of avg@32 / pass@32 (\%) across six benchmarks. Our default \EAPO{} configuration is {\setlength{\fboxsep}{0.5pt}\colorbox{ours!15}{shaded}}. Best scores are \textbf{bolded}.}
        \label{tab:sign_entropy_coupling}
        \vspace{-0.1in}
        \footnotesize
        \setlength{\tabcolsep}{1pt}
        \renewcommand{\arraystretch}{1.25}
        \begin{tabular*}{\linewidth}{@{\extracolsep{\fill}}clccc@{}}
            \toprule
            & \textbf{Entropy} & \multicolumn{3}{c}{\textbf{\textcolor{failure}{Penalization}} ($\hat{A}^i<0$): $b_{-}$} \\
            & \textbf{Preference}
            & Low ($-1$) & Uniform ($0$) & High ($+1$) \\
            \midrule
            \multirow{3}{*}{\shortstack{\textbf{\textcolor{success}{Reinforcement}}\\($\hat{A}^i>0$): $b_{+}$}}
            & High ($+1$)
            & \cellcolor{ours!15}\textbf{31.00\,/\,61.27}
            & 26.73\,/\,53.83
            & 25.36\,/\,52.09 \\
            & Uniform ($0$)
            & 26.08\,/\,51.26
            & 24.48\,/\,50.43
            & 22.89\,/\,50.03 \\
            & Low ($-1$)
            & 23.99\,/\,49.84
            & 23.07\,/\,48.62
            & 22.56\,/\,47.83 \\
            \bottomrule
        \end{tabular*}
    \end{minipage}
\end{table}

%% file: tables/kappa_sensitivity.tex
\begin{wraptable}[8]{r}{0.36\textwidth}
    \vspace{-0.18in}
    \centering
    \caption{\textbf{Sensitivity to $\kappa$.} We report avg@32 / pass@32 (\%) across different values of $\kappa$.}
    \label{tab:kappa_sensitivity}
    \vspace{-0.12in}
    \footnotesize
    \setlength{\tabcolsep}{1pt}
    \renewcommand{\arraystretch}{0.95}
    \begin{tabular*}{\linewidth}{@{\extracolsep{\fill}}lcc@{}}
        \toprule
        $\kappa$ & \textbf{AIME25} & \textbf{AIME26} \\
        \midrule
        $0$ & 10.42\,/\,36.67 & 7.50\,/\,33.33 \\
        $\log 2$ & 10.73\,/\,33.33 & 10.63\,/\,26.67 \\
        $\log 4$ & 17.08\,/\,\textbf{46.67} & 15.73\,/\,\textbf{50.00} \\
        $\log 8$ & \textbf{18.23}\,/\,\textbf{46.67} & 15.83\,/\,40.00 \\
        $\log 16$ & 15.73\,/\,43.33 & \textbf{16.15}\,/\,46.67 \\
        \bottomrule
    \end{tabular*}
    \vspace{-0.1in}
\end{wraptable}

%% file: sections/6_related_work.tex
\section{Related Work}
\label{sec:related_work}

\paragraph{Fine-Grained Credit Assignment.}
Recent work supplements sparse outcome-level rewards in RLVR with finer-grained signals from learned process rewards~\citep{math-shepherd, prime}, intermediate value estimates~\citep{vineppo, spo}, or teacher- and oracle-conditioned likelihoods~\citep{rlsd, oppo}. These often require auxiliary models, rollouts, or privileged information. Other approaches use the policy's internal signals, deriving token importance from outcome sensitivity~\citep{oar, grail} or combining entropy with correctness~\citep{less}. Token-level advantages are also reweighted using selected-token probabilities or surprisal to modulate policy updates~\citep{a3po, opsa}.

\paragraph{RLVR for Exploratory Reasoning.}
To broaden exploration in RLVR, existing approaches allocate more samples to difficult problems~\citep{dars}, revisit promising states~\citep{rrl}, or branch over alternative continuations~\citep{deepsearch, treeadv}. Explicit diversity objectives encourage complementary outcomes and reasoning traces~\citep{outcomeexploration, diver, vpo}. Other methods regulate policy concentration~\citep{cov-entropy, steer, papo} or preserve low-probability alternatives~\citep{lpreg}. Token-level credit assignment supports exploration through reversed teacher signals~\citep{rlrt}, updates focused on high-entropy tokens~\citep{forking-tokens-8020, hapo}, or entropy bonuses added to token advantages~\citep{entropyadv}. However, these entropy-based methods incorporate uncertainty in an outcome-agnostic manner, either leaving uncertain positions in failed responses most heavily penalized or confident failures insufficiently corrected. In contrast, \EAPO{} couples policy entropy with the advantage sign for asymmetric credit assignment, reinforcing uncertain successes while concentrating correction on confident failures and preserving exploratory alternatives, without requiring auxiliary models, additional rollouts, or privileged information.

%% file: sections/7_conclusion.tex
\section{Conclusion}
\label{sec:conclusion}

In this work, we proposed \EAPO{}, an entropy-guided credit assignment method for improving exploration in LLM reasoning. By coupling entropy with the advantage sign, \EAPO{} reinforces uncertain decisions in successful responses and penalizes confident decisions in unsuccessful responses, while attenuating penalties on uncertain alternatives. Evaluations on reasoning tasks demonstrate consistent overall improvements over RLVR baselines designed to promote exploration, across both base and reasoning models. Our analyses further show that \EAPO{} effectively broadens problem coverage and increases response diversity. These findings highlight the value of asymmetric credit assignment for reinforcing \surprisingsuccess{} and correcting \repeatedfailure{} while sustaining exploration.

\input{sections/X_limitations}

\section*{Ethics Statement}

This work aims to improve exploration in LLM reasoning using existing models and reasoning datasets. However, models trained with \EAPO{} may retain biases from their underlying models and data, and improved reasoning performance does not guarantee reliable or safe outputs. We therefore encourage appropriate safeguards when applying the method beyond the reasoning tasks studied here or in real-world applications.

%% file: sections/X_limitations.tex
\section*{Limitations and Future Directions}
\label{sec:limitations}

\EAPO{} uses the policy's own entropy as an intrinsic signal for token-level credit assignment, allowing it to guide exploration without auxiliary models, additional information, or extra compute. However, this signal reflects the model's learned preferences and confidence, making the resulting allocation sensitive to its existing biases. The reliability of entropy-guided credit assignment therefore depends on how well the model's uncertainty aligns with the actual contribution of individual reasoning decisions to the final outcome.

While our work effectively expands exploration in LLM reasoning, our analysis of the underlying mechanism focuses on alternative continuations from a given reasoning state, leaving the explicit combination of findings across different states or trajectories unexplored. Recent advances in scientific discovery, however, illustrate the value of maintaining diverse candidate solutions and iteratively reusing or combining their useful components to generate new hypotheses and algorithms~\citep{alphaevolve, ai-co-scientist}. Motivated by these approaches, an important direction for future work is to investigate how exploration can be strengthened in such iterative discovery settings.

%% file: sections/appendix/1_extended_related_work.tex
\section{Extended Related Work}
\label{sec:extended_related_work}

\paragraph{Uncertainty-Guided RLVR.}
Many existing RLVR approaches concentrate learning on high-entropy regions or seek to increase policy entropy to promote exploration. Specifically, 80/20~\citep{forking-tokens-8020} restricts policy gradients to the highest-entropy 20\% of tokens, while HAPO~\citep{hapo} increases advantage magnitudes at high-entropy positions under both advantage signs. EntropyAdv~\citep{entropyadv} promotes exploration by adding a nonnegative entropy bonus to token advantages regardless of outcome, which only relieves penalties at uncertain positions and leaves confident decisions in failed responses with no stronger correction than under uniform credit. Recent approaches also regulate updates to low-probability tokens, protecting useful alternatives from excessive suppression~\citep{lpreg} or suppressing low-probability branches to favor more confident reasoning alternatives~\citep{opsa}. While these methods shape credit at the token level, they incorporate uncertainty in an outcome-agnostic manner. UCAS~\citep{ucas} and LESS~\citep{less}, on the other hand, distinguish success from failure but at coarser granularities. Specifically, UCAS scales the advantage of each response by its sequence-level confidence depending on the advantage sign, with only a small certainty-based penalty at the token level, while LESS targets confident reasoning patterns in low-entropy segments. In contrast, \EAPO{} couples policy entropy with the response advantage sign at the token level to reverse the entropy preference between reinforcement and penalization, strengthening uncertain successes while concentrating correction on confident positions in failed responses to prevent their recurrence and attenuating penalties at uncertain positions to preserve exploratory alternatives.

\paragraph{Outcome-Conditioned Credit Assignment.}
Several recent approaches also seek to promote exploration by assigning credit differently to successful and failed responses. GEPO~\citep{gepo} attenuates positive advantages in low-entropy response groups to reduce over-exploitation and negative advantages in high-entropy groups to preserve exploration, while leaving token-level credit uniform within each response. RLRT~\citep{rlrt} reweights token-level advantages using reversed teacher--student likelihood ratios to reinforce successful exploration, but requires privileged context and additional teacher forward passes while leaving failed responses uniformly weighted. A3PO~\citep{a3po} leverages sampled-token probability to amplify advantage magnitudes for selected tokens, leaving the advantages of the remaining tokens unchanged. However, this can increase the overall advantage magnitude without attenuating penalties on potentially exploratory decisions, and sampled-token probability does not fully characterize predictive uncertainty, as replacing entropy with surprisal destabilizes asymmetric credit allocation (\Cref{app:entropy_vs_surprisal}). Our approach, \EAPO{}, uses full-vocabulary policy entropy to redistribute advantage across all token positions while preserving its within-response mean. By conditioning token weights on entropy and the response advantage sign, \EAPO{} both reinforces uncertain successes and preserves exploratory alternatives in failed responses without privileged context or additional teacher forward passes.

%% file: sections/appendix/2_add_exp_details.tex
\clearpage
\section{Additional Experimental Details}
\label{app:add_exp_details}

\subsection{Additional Details on the Resampling Experiment}
\label{app:resampling_details}

We provide further details on the experimental setup and selection statistics for the high- and low-entropy windows used in the resampling experiment in \Cref{sec:motivation}. Specifically, for each of Qwen3-4B-Base and Qwen3-8B-Base, we select 20 moderately difficult problems from AIME24/25/26, HMMT26, and AMC23, where moderate difficulty is defined as 6--14 correct responses out of 32 attempts by the corresponding model. For each problem, we randomly select two correct and two incorrect responses from these attempts. Within each selected response, we exclude the first and last 20\% of tokens and consider pairs of non-overlapping 32-token windows whose starting positions differ by at most 20\% of the response length. We select the pair with the largest difference in mean token entropy. We then fix the original prefix preceding each selected window and sample 64 continuations, measuring overlap as the mean pairwise common-prefix length over their first 32 generated tokens and reward as their final-answer accuracy. Both the initial responses and the continuations use the same sampling settings as training rollouts (\Cref{tab:training_hyperparameters}) with a random seed of 42 for both high- and low-entropy windows.

\input{tables/resampling_positions}
We additionally analyze the relative starting positions of the selected high- and low-entropy windows for each model, expressed as a percentage of the original response length. \Cref{tab:resampling_positions} reports their means and standard deviations. Across both models, the high- and low-entropy windows have mean starting positions near the middle of the responses, with neither type consistently preceding the other, suggesting that window selection does not systematically favor earlier high-entropy windows or later low-entropy windows.

\subsection{Benchmark Details}
\label{app:benchmark_details}

We evaluate on 297 problems across six mathematical reasoning benchmarks: AIME24, AIME25, and AIME26~\citep{aime}, with 30 problems each; HMMT26~\citep{hmmt}, with 33 problems; AMC23~\citep{amc}, with 40 problems; and MATH500-H~\citep{math500-verify-step-by-step}, the level-5 subset of MATH500 containing 134 problems. We automatically grade the generated answers against the reference answers using Math-Verify~\citep{math-verify}. We sample 32 responses per mathematical reasoning problem and report avg@32 as the mean response accuracy and pass@32 as the probability of generating at least one correct response, using the unbiased estimator of \citet{unbiased-passk} for pass@$k$. We average scores over problems within each benchmark and give equal weight to the six benchmarks when reporting an overall mean.

To assess out-of-distribution generalization, we evaluate on eight tasks from Reasoning Gym~\citep{reasoning-gym}, covering logical, spatial, and algorithmic reasoning. The task headers in \Cref{tab:ood_results} correspond to graph coloring (Graph), color cube rotation (Cube), mini sudoku (Sudoku), family relationships (Family), knights \& knaves (K\&K), group anagrams (Anag.), zebra puzzles (Zebra), and palindrome generation (Palin.). We use the ``easy'' task configurations by default, but modify the settings for three tasks: graph coloring uses 12--14 vertices and an edge probability of 0.15, mini sudoku uses 4--6 empty cells, and zebra puzzles use three people and three characteristics. For each task, we evaluate 50 problems, sampling 16 responses per problem and using the task-specific verifier to assess full correctness, without counting partial credit as success. We report avg@16 within each task and compute the macro average by assigning equal weight to all eight tasks before rounding.

\subsection{Training Hyperparameters}
\label{app:training_hyperparameters}

\Cref{tab:training_hyperparameters} summarizes the shared training hyperparameters for \EAPO{} and the baselines, including the optimizer, batching, generation, and adapter settings. We train on DAPO-Math-17k-Processed using a DAPO-style configuration~\citep{dapo}, including asymmetric clipping, token-level loss aggregation, and a soft overlong penalty. Unless otherwise stated, all optimization-based methods use LoRA~\citep{lora} and are optimized with AdamW~\citep{adamw}. We implement training with TRL~\citep{trl} and use vLLM~\citep{vllm} for rollout generation. All experiments are conducted using a single NVIDIA H200 GPU. Parameters not listed in the table follow the defaults of the corresponding implementations.

\input{tables/training_hyperparameters}

\clearpage
\subsection{Details on Strategy Coverage}
\label{app:strategy_coverage_details}

We use all 123 problems from AIME24/25/26 and HMMT26 with the 32 responses per problem from the main Qwen3-4B-Base experiments for all seven methods, and shuffle the summaries of each problem with a fixed seed before clustering. \Cref{fig:prompt_strategy_summary,fig:prompt_strategy_cluster} show the prompts for the two stages. We compute the expectation in \Cref{eq:strategy_coverage} exactly as $\mathbb{E}\left[|S_{p,m}(k)|\right]=\sum_{s}\big(1-\binom{32-n_s}{k}/\binom{32}{k}\big)$, where $n_s$ is the number of correct responses of method $m$ with strategy $s$. To assess the reliability of the strategy labels, we repeat the second stage with Claude Opus 5.5~\citep{opus5-5} on the identical summaries, which shows substantial agreement with GPT-6.1-Sol, with a mean Adjusted Rand Index~\citep{ari} of 0.72 across problems.

\input{figures/prompt_strategy_summary}
\input{figures/prompt_strategy_cluster}

%% file: tables/resampling_positions.tex
\begin{wraptable}{r}{0.39\textwidth}
    \vspace{-0.18in}
    \centering
    \caption{Starting positions of selected windows (\% of response length; mean $\pm$ std.).}
    \label{tab:resampling_positions}
    \vspace{-0.1in}
    \small
    \setlength{\tabcolsep}{1pt}
    \renewcommand{\arraystretch}{1.05}
    \begin{tabular*}{\linewidth}{@{\extracolsep{\fill}}lcc@{}}
        \toprule
        \textbf{Model} & \textbf{High entropy} & \textbf{Low entropy} \\
        \midrule
        Qwen3-4B-Base & $47.29\pm6.20$ & $53.73\pm7.11$ \\
        Qwen3-8B-Base & $51.44\pm8.02$ & $49.86\pm6.42$ \\
        \bottomrule
    \end{tabular*}
    \vspace{-0.08in}
\end{wraptable}

%% file: tables/training_hyperparameters.tex
\begin{table}[p]
    \centering
    \caption{Training hyperparameters shared by \EAPO{} and the baselines. Parameters not listed here follow the defaults of the original implementations.}
    \label{tab:training_hyperparameters}
    \vspace{-0.1in}
    \small
    \setlength{\tabcolsep}{6pt}
    \renewcommand{\arraystretch}{1.1}
    \begin{tabular}{@{}>{\raggedright\arraybackslash}p{0.17\linewidth}>{\raggedright\arraybackslash}p{0.39\linewidth}>{\raggedright\arraybackslash}p{\dimexpr0.44\linewidth-4\tabcolsep\relax}@{}}
        \toprule
        \textbf{Category} & \textbf{Parameter} & \textbf{Value} \\
        \midrule
        Data & Maximum response length & 10,240 (base); 38,912 (reasoning) \\
             & Soft overlong penalty onset & 8,192 (base); 32,768 (reasoning) \\
        \midrule
        Batching & Prompts per rollout iteration & 16 \\
                 & Responses per prompt ($G$) & 8 \\
                 & Generation batch size & 128 \\
                 & Mini-batch size & 64 \\
                 & Rollout iterations & 100 \\
        \midrule
        Optimization & Optimizer & AdamW \\
                     & Learning rate & $1\times 10^{-5}$ \\
                     & Learning-rate schedule & Constant with warmup \\
                     & Warmup steps & 10 \\
                     & Weight decay & $0.01$ \\
                     & Maximum gradient norm & $1.0$ \\
                     & Random seed & 42 \\
        \midrule
        LoRA & Rank ($r$) & 32 \\
             & Alpha ($\alpha$) & 64 \\
             & Dropout & $0$ \\
             & Target modules & Attention and MLP projections \\
        \midrule
        Generation & Inference engine & vLLM \\
                   & Temperature / top-$p$ & $1.0$ / $1.0$ \\
                   & Top-$k$ filtering & Disabled \\
        \midrule
        Policy loss & Clipping thresholds $(\epsilon_{\mathrm{low}},\epsilon_{\mathrm{high}})$ & $(0.2, 0.28)$ \\
                    & Loss aggregation & Token mean \\
                    & KL penalty coefficient & $0$ \\
                    & Entropy bonus coefficient & $0$ \\
                    & Importance sampling & Token-level TIS (cap $=2.0$) \\
                    & Advantage std normalization & Group \\
        \midrule
        \EAPO{} & Concentration ($\kappa$) & $\log 4$ \\
                & $\varepsilon_H$          & $10^{-8}$ \\
        \bottomrule
    \end{tabular}
\end{table}

%% file: figures/prompt_strategy_summary.tex
\begin{figure}[h]
    \centering
    \begin{minipage}{\linewidth}
        \fontsize{9}{11.5}\selectfont
        \setlength{\parindent}{0pt}
        \setlength{\parskip}{0pt}

        \begin{qualitativebox}[colframe=qualitativeframe,boxrule=0.55pt,top=8pt,bottom=8pt,before skip=0pt]
        \begin{qualitativebox}[colback=qualitativeproblem,boxrule=0pt,top=5pt,bottom=5pt,before skip=0pt]
            \textbf{System Prompt}
        \end{qualitativebox}

        \begin{qualitativerow}
            {\raggedright
            You analyze a model's attempt at a competition math problem. Describe the solution \textbf{strategy} only: the core mathematical idea and how the problem is decomposed (e.g., ``set up a recurrence and find its period mod 7'', ``coordinate bash with the circle centered at the origin''). Ignore arithmetic details, notation, ordering of routine steps, verification, and writing style. Do \textbf{not} judge correctness and do not mention whether the final answer is right. If the attempt tries several approaches, describe only the final approach it commits to for its final answer.
            \par\vspace{5pt}
            Return a 1--2 sentence description of the final approach.
            \par}
        \end{qualitativerow}

        \begin{qualitativebox}[colback=qualitativeproblem,boxrule=0pt,top=5pt,bottom=5pt,before skip=0pt]
            \textbf{User Prompt}
        \end{qualitativebox}

        \begin{qualitativerow}
            \textbf{Problem:} \textcolor{qualitativemuted}{\textit{$\langle$problem statement$\rangle$}}
            \par\vspace{4pt}
            \textbf{Attempt:} \textcolor{qualitativemuted}{\textit{$\langle$sampled response$\rangle$}}
        \end{qualitativerow}
        \end{qualitativebox}
    \end{minipage}
    \vspace{-0.05in}
    \caption{\textbf{Prompt for strategy summarization.} The judge summarizes the final solution strategy of each sampled response independently, without access to its method or correctness.}
    \label{fig:prompt_strategy_summary}
\end{figure}

%% file: figures/prompt_strategy_cluster.tex
\begin{figure}[h]
    \centering
    \begin{minipage}{\linewidth}
        \fontsize{9}{11.5}\selectfont
        \setlength{\parindent}{0pt}
        \setlength{\parskip}{0pt}

        \begin{qualitativebox}[colframe=qualitativeframe,boxrule=0.55pt,top=8pt,bottom=8pt,before skip=0pt]
        \begin{qualitativebox}[colback=qualitativeproblem,boxrule=0pt,top=5pt,bottom=5pt,before skip=0pt]
            \textbf{System Prompt}
        \end{qualitativebox}

        \begin{qualitativerow}
            {\raggedright
            You group solution strategies for one competition math problem. You get the problem and a list of strategy descriptions, each with an id. Group descriptions that share the same core mathematical idea or problem decomposition. Differences in computation order, notation, verification, or level of detail do \textbf{not} make strategies different. Do not consider whether any attempt is correct. Every id must be assigned to exactly one group.
            \par\vspace{5pt}
            Return JSON: \{``strategies'': [\{``label'': ``S1'', ``description'': \textcolor{qualitativemuted}{\textit{$\langle$one sentence$\rangle$}}\}, \ldots], ``assignments'': \{\textcolor{qualitativemuted}{\textit{$\langle$id$\rangle$}}: ``S1'', \ldots\}\}
            \par}
        \end{qualitativerow}

        \begin{qualitativebox}[colback=qualitativeproblem,boxrule=0pt,top=5pt,bottom=5pt,before skip=0pt]
            \textbf{User Prompt}
        \end{qualitativebox}

        \begin{qualitativerow}
            \textbf{Problem:} \textcolor{qualitativemuted}{\textit{$\langle$problem statement$\rangle$}}
            \par\vspace{4pt}
            \textbf{Strategy descriptions:}
            \par\vspace{2pt}
            a1: \textcolor{qualitativemuted}{\textit{$\langle$final approach of summary 1$\rangle$}}
            \par
            a2: \textcolor{qualitativemuted}{\textit{$\langle$final approach of summary 2$\rangle$}}
            \par
            \qualitativeomitted
        \end{qualitativerow}
        \end{qualitativebox}
    \end{minipage}
    \vspace{-0.05in}
    \caption{\textbf{Prompt for strategy clustering.} For each problem, the summaries from all methods are shuffled and anonymized, and the judge partitions them into strategies without access to method identity or correctness.}
    \label{fig:prompt_strategy_cluster}
\end{figure}

%% file: sections/appendix/3_add_exp_results.tex
\clearpage
\section{Additional Experimental Results}
\label{app:add_exp_results}

\subsection{Comparison across Credit Assignment Granularities}
\label{app:additional_related_methods}

\input{tables/related_methods}

We further compare \EAPO{} with UCAS~\citep{ucas}, GEPO~\citep{gepo}, and LESS~\citep{less}, adjacent methods that also condition credit on uncertainty or outcome but at the level of responses, response groups, and segments, respectively (\Cref{sec:extended_related_work}). In contrast, \EAPO{} redistributes the response advantage across individual tokens according to their entropy and the advantage sign. We train each method on Qwen3-4B-Base under the same training setup as the main experiments and evaluate on AIME24/25/26 and HMMT26. As shown in \Cref{tab:related_methods}, \EAPO{} outperforms all three methods on every benchmark by a large margin in both avg@32 and pass@32, suggesting that asymmetric credit assignment at the token level is central to its gains.

\subsection{Entropy vs. Surprisal}
\label{app:entropy_vs_surprisal}

We examine whether sampled-token surprisal can substitute for policy entropy in \EAPO{}. Specifically, the surprisal-based variant uses the normalized negative log-probability of the sampled token,
\begin{equation}
    \mathcal{S}_{i,t}
    = -\log\pi_{\mathrm{old}}(y_{i,t}\mid x,y_{i,<t}),
    \label{eq:token_surprisal}
\end{equation}
in place of entropy as the signal for the sign-dependent credit allocation in \Cref{eq:entropic_allocation}.

\input{figures/surprisal_s1_size}

\Cref{fig:surprisal} shows the training dynamics of entropy-based \EAPO{} and its surprisal-based variant on Qwen3-4B-Base. \EAPO{} continues to improve its training reward, whereas the surprisal-based variant shows declining reward after a few training steps and eventually collapses to zero reward. Although the expectation of surprisal under the policy equals entropy at a fixed prefix, sampling a low-probability token can yield high surprisal even under a concentrated policy. Consequently, surprisal-based reweighting may conflate rare token selections with uncertain decisions, potentially destabilizing credit allocation. In contrast, entropy summarizes the full next-token distribution independently of which token is sampled. Consistent with this, we found a similar collapse of the training reward when reproducing A3PO~\citep{a3po} under our setting, using its original settings, while matching only the hyperparameters shared by all methods. These results suggest that the sampled-token surprisal is less suitable than entropy for dense asymmetric advantage redistribution.

\subsection{Recovery from Low-Entropy Windows}
\label{app:recovery_low_entropy}

\input{tables/recovery_intervention}We further examine whether failed responses, which tend to remain incorrect when resampled from low-entropy windows (\Cref{sec:motivation}), can still be recovered from the prefixes preceding these windows. Specifically, we insert a short reflection cue (e.g., ``Wait, let me reconsider this step.'') at the start of low-entropy windows in newly drawn failed responses, following the setup in \Cref{app:resampling_details}, and compare 64 continuations from the same prefix with and without the cue. As shown in \Cref{tab:recovery_intervention}, the cue improves accuracy on both models, indicating that a correct continuation remains reachable from these prefixes once the policy departs from its confident behavior. Without the cue, however, the policy largely regenerates the same path (\Cref{tab:regen_entropy}), where \EAPO{} concentrates penalties in failed responses to steer the policy away from such confident continuations, preventing not only the recurrence of failure patterns but also confident continuation from erroneous prefixes.
\vspace{0.1in}

\subsection{Extended Ablation on Sign--Entropy Coupling}
\label{app:ext_se_couple}

We provide an additional ablation on sign--entropy coupling with Qwen3-8B-Base, extending the analysis in \Cref{sec:ablation_studies}. As shown in \Cref{tab:sign_entropy_coupling_8b}, performance improves as the entropy preference shifts from high to low for penalization and from low to high for reinforcement, mirroring the trends observed with Qwen3-4B-Base. Specifically, with high-entropy reinforcement fixed ($b_{+}=+1$), shifting penalization from high to low entropy improves avg@32 by 4.25 percentage points. Similarly, with low-entropy penalization fixed ($b_{-}=-1$), shifting reinforcement from low to high entropy improves avg@32 by 6.22 percentage points, indicating that both directions contribute to the gains of \EAPO{}. Consequently, the default \EAPO{} configuration $(b_{+},b_{-})=(+1,-1)$, which combines both preferences, achieves the highest avg@32 and pass@32 among all nine combinations. These results further support \EAPO{}'s asymmetric allocation directions for reinforcement and penalization.

\input{tables/coupling_8b_recurrence}

\subsection{Recurrence of Successful and Failed Strategies}
\label{app:strategy_recurrence}

To examine whether surprising successes of the initial policy are reproduced and its confident failures no longer recur after training, we label each response of the initial Qwen3-4B-Base on AIME24/25/26 and HMMT26 as high- or low-entropy based on the entropy at its critical decision points. Using the strategy labels in \Cref{sec:exploration_analysis}, we then measure how often each trained policy reproduces the strategy of a high-entropy success with a correct answer, or that of a low-entropy failure with an incorrect answer. As shown in \Cref{tab:strategy_recurrence}, \EAPO{} reproduces the strategies behind high-entropy successes most often among all methods, while reducing the recurrence of low-entropy failures far more than all baselines. Notably, EntropyAdv, whose entropy bonus attenuates penalties at uncertain positions without strengthening the correction of confident decisions, shows a failure recurrence comparable to GRPO. These results are consistent with the gains from shifting penalization toward low-entropy positions in \Cref{tab:sign_entropy_coupling,tab:sign_entropy_coupling_8b}, indicating that \EAPO{} makes surprising success more repeatable and corrects repeated failure.

\subsection{Distinct Strategies under Matched Correct Responses}
\label{app:strategy_within_problem}

To isolate whether \EAPO{} truly generates more diverse successful reasoning strategies within each problem, not merely solves more problems, we match both the problems and the number of correct responses across methods. Specifically, on AIME24/25/26 and HMMT26 with Qwen3-4B-Base, we sample responses of each method until 32 correct ones are obtained (with a maximum of 256 samples; reusing the 256 responses per problem in \Cref{sec:exploration_analysis,app:passk_other_benchmarks}), and consider only the problems on which all methods reach 32 correct responses, obtaining 27 problems in total. We then count the distinct strategies among these correct responses, following the strategy labeling in \Cref{sec:exploration_analysis}. As shown in \Cref{tab:strategy_within_problem}, \EAPO{} reaches correct answers through the largest number of distinct strategies per problem, whereas GRPO, EntropyAdv, HAPO, and 80/20 remain close to the initial policy. These results indicate that \EAPO{} discovers more distinct successful strategies within each problem, beyond solving more problems.

\input{tables/strategy_within_problem}

\subsection{Performance under Matched Token Budgets}
\label{app:token_budget}

We examine whether longer responses alone explain the performance gains of \EAPO{}, extending baseline generations with \emph{budget forcing} following \citet{s1}. Specifically, we evaluate Qwen3-4B-Base on all 30 AIME24 problems with 32 responses per problem, comparing the initial model, GRPO, EntropyAdv, HAPO, and RLRT against \EAPO{} without forced continuations. For each baseline, we remove the termination token (or \texttt{</think>}) and append \texttt{Wait} to resume generation, allowing up to $K \in \{0,1,2,3,4\}$ continuations under a total output budget of 32,768 tokens. We use the same sampling settings as during training.

\Cref{fig:s1_resp_length} suggests that additional continuations do not necessarily translate into more effective exploration. Specifically, \EAPO{} achieves an avg@32 of 20.2\% with a response length of 5.3k tokens, whereas the strongest baseline, RLRT, peaks at 15.7\% with a mean response length of approximately 6.2k tokens. Most approaches show little benefit from early continuations and lose accuracy as responses are extended further. These results suggest that the performance gain of \EAPO{} stems from enhanced exploration, with increased response length arising as a consequence of this process.

\subsection{Scaling with Model Size}
\label{app:model_size}

To examine the scaling behavior of \EAPO{} as model size increases, we evaluate the Qwen3-Base family from 1.7B to 14B parameters, reporting macro-averaged avg@32 across the six mathematical reasoning benchmarks with $\pm 1$ standard error bars.\footnote{We compute standard error as $B^{-1}\sqrt{\sum_{b=1}^{B}s_b^2/N}$, where $B$ is the number of benchmarks (6), $N$ is the number of generations (32), and $s_b$ is the sample standard deviation of benchmark $b$'s accuracy over generations.} As shown in \Cref{fig:model_size}, performance improves with model size, while the relative performance trends among methods remain broadly consistent across scales, with \EAPO{} achieving the highest mean accuracy throughout. These results suggest that the benefits of asymmetric entropy-guided credit assignment persist as model capacity increases.

\subsection{LoRA vs. Full Fine-Tuning}
\label{app:lora_vs_full}

\input{tables/lora_vs_full}

Exploration can depend on retaining useful low-probability alternatives, motivating us to examine whether the benefits of entropy-guided credit assignment depend on the fine-tuning strategy. Our main experiments use LoRA~\citep{lora}, which constrains weight updates to a low-rank form and may therefore produce different changes in token probabilities from full fine-tuning. To assess whether \EAPO{}'s gains also hold under full fine-tuning, we compare \EAPO{} and the baselines under LoRA and full fine-tuning on Qwen3-4B-Base across the six mathematical reasoning benchmarks. To ensure comparable training conditions, we use a learning rate of $1\times 10^{-6}$ for full fine-tuning while keeping all other experimental settings identical to those used for LoRA.

\Cref{tab:lora_vs_full} shows that LoRA and full fine-tuning yield similar overall performance, with absolute differences in macro-averaged avg@32 of less than one percentage point on average across methods. Moreover, \EAPO{} maintains the best overall performance under both strategies, with its margin over the strongest baseline exceeding these shifts in performance. These results suggest that the benefits of asymmetric entropy-guided credit assignment are robust to the choice of fine-tuning strategy.

\subsection{Performance across Random Seeds}
\label{app:random_seeds}

\input{figures/random_seeds}
To evaluate whether the performance gains of \EAPO{} are robust to random seed selection, we evaluate GRPO, EntropyAdv, HAPO, 80/20, RLRT, and \EAPO{} on Qwen3-4B-Base and Qwen3-4B across three distinct random seeds, extending our main experiment (which used a fixed seed of 42). \Cref{fig:random_seeds} reports the mean training reward and standard deviation across the three runs for each backbone, with a 15-step moving average applied to the curves. On Qwen3-4B-Base \figleft{}, \EAPO{} maintains higher mean reward than all baselines over most of training, with only RLRT, the strongest baseline, marginally exceeding it only around steps 150--170, and the gap becomes particularly pronounced in the final stages, where its standard-deviation band largely lies above those of the baselines. This separation suggests that the reward improvement is substantial relative to the observed variation across seeds. On Qwen3-4B \figright{}, where all methods start from a high reward due to extensive prior post-training, \EAPO{} closely tracks the baselines early in training and achieves the highest mean reward over most of the latter half, with the gap widening toward the end, while variation across seeds remains small for all methods. Overall, these results demonstrate the consistency of \EAPO{}'s gains beyond single-seed training on both base and reasoning backbones.

\input{tables/random_seeds_train}
To examine whether this consistency extends to held-out performance, we further evaluate the final checkpoints of the three Qwen3-4B-Base training runs on AIME24/25/26 and HMMT26, using the same evaluation settings as in \Cref{tab:main_results}. \Cref{tab:random_seeds_train} reports the mean and standard deviation of avg@32 and pass@32 across the three runs. \EAPO{} achieves the best mean avg@32 and pass@32 on all four benchmarks with avg@32 gaps over the strongest baseline exceeding the standard deviations of both \EAPO{} and that baseline, confirming that its held-out gains are also consistent across training runs.

Beyond variation across training runs, we further examine whether the gains in \Cref{tab:main_results} are robust to the randomness in evaluation sampling. Specifically, we evaluate the checkpoints from our main experiments for all methods and backbones under three distinct evaluation seeds on AIME24/25/26 and HMMT26, while keeping all other evaluation settings identical. \Cref{tab:random_seeds_eval} shows that the variation across evaluation seeds is small relative to the differences between methods, and \EAPO{} achieves the best mean avg@32 on most benchmarks across backbones, indicating that the gains in \Cref{tab:main_results} do not stem from a particular evaluation seed. Notably, even for reasoning backbones, where gaps between methods are narrower due to extensive prior post-training, \EAPO{} outperforms the strongest baseline in most cases with similarly small variation across seeds.

\subsection{Pass@$k$ on Additional Models and Benchmarks}
\label{app:passk_other_benchmarks}

To examine whether the coverage gains in \Cref{fig:exploration} (Left) extend beyond a single backbone and benchmark, we further evaluate pass@$k$ on HMMT26 with Qwen3-4B-Base and on AIME26 and HMMT26 with Qwen3-8B-Base, following the same sampling procedure. As shown in \Cref{fig:passk_other_benchmarks}, \EAPO{} achieves the highest pass@$k$ across all sampling budgets in all three settings, with its margin over the baselines persisting at large $k$. These results suggest that the broader problem coverage of \EAPO{} is not specific to a particular backbone or benchmark.

\input{tables/random_seeds_eval}
\input{figures/passk_other_benchmarks}

%% file: tables/related_methods.tex
\begin{table}[t]
    \centering
    \caption{Comparison with methods assigning credit at coarser granularities on Qwen3-4B-Base across four mathematical reasoning benchmarks. \textbf{Bold} denotes the best performance.}
    \label{tab:related_methods}
    \vspace{-0.1in}
    \small
    \setlength{\tabcolsep}{3pt}
    \resizebox{\linewidth}{!}{%
    \begin{tabular}{p{3.1cm}*{8}{>{\centering\arraybackslash}p{1.2cm}}}
        \toprule
        \multirow{2}{*}{\textbf{Method}}
        & \multicolumn{2}{c}{\textbf{AIME24}}
        & \multicolumn{2}{c}{\textbf{AIME25}}
        & \multicolumn{2}{c}{\textbf{AIME26}}
        & \multicolumn{2}{c}{\textbf{HMMT26}} \\
        \cmidrule(lr){2-3}\cmidrule(lr){4-5}\cmidrule(lr){6-7}\cmidrule(lr){8-9}
        & {\scriptsize Avg@32} & {\scriptsize Pass@32}
        & {\scriptsize Avg@32} & {\scriptsize Pass@32}
        & {\scriptsize Avg@32} & {\scriptsize Pass@32}
        & {\scriptsize Avg@32} & {\scriptsize Pass@32} \\
        \midrule
        UCAS & 13.4 & 40.0 & 11.8 & 36.7 & 9.3 & 33.3 & 7.4 & 30.3 \\
        GEPO & 12.8 & 36.7 & 11.3 & 36.7 & 8.4 & 30.0 & 6.9 & 27.3 \\
        LESS & 14.6 & 43.3 & 12.9 & 40.0 & 10.9 & 36.7 & 8.1 & 30.3 \\
        \rowcolor{ours!15}
        \textbf{\EAPO{} (Ours)}
             & \textbf{20.2} & \textbf{53.3} & \textbf{17.1} & \textbf{46.7} & \textbf{15.7} & \textbf{50.0} & \textbf{13.4} & \textbf{33.3} \\
        \bottomrule
    \end{tabular}%
    }
\end{table}

%% file: figures/surprisal_s1_size.tex
\begin{figure}[t]
    \vspace{0.1in}
    \centering
    \includegraphics[width=\linewidth,clip]{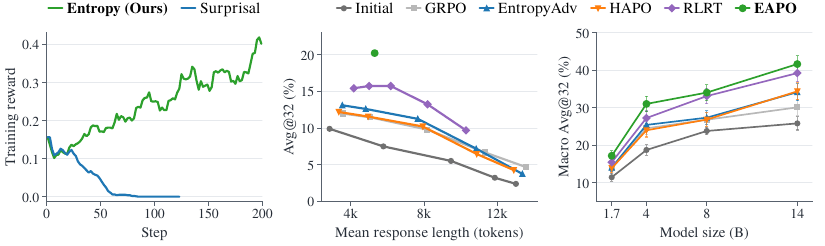}
    \par\vspace{-0.05in}
    \captionsetup{margin={0pt,5pt},justification=raggedright,singlelinecheck=false}
    \begin{minipage}[t]{\dimexpr\linewidth/3\relax}
        \vspace{-0.05in}
        \centering
        \captionof{figure}{Entropy vs. Surprisal\\as signals for credit allocation.}
        \label{fig:surprisal}
        \vspace{-0.1in}
    \end{minipage}%
    \hfill
    \begin{minipage}[t]{\dimexpr\linewidth/3\relax}
        \vspace{-0.05in}
        \centering
        \captionsetup{margin={1.5bp,5pt}}
        \captionof{figure}{Response length vs.\\accuracy with budget forcing.}
        \label{fig:s1_resp_length}
        \vspace{-0.1in}
    \end{minipage}%
    \hfill
    \begin{minipage}[t]{\dimexpr\linewidth/3\relax}
        \vspace{-0.05in}
        \centering
        \captionsetup{margin={3bp,3pt}}
        \captionof{figure}{Scaling with model\\size in the Qwen3-Base family.}
        \label{fig:model_size}
        \vspace{-0.1in}
    \end{minipage}%
\end{figure}

%% file: tables/recovery_intervention.tex
\begin{wraptable}{r}{0.47\textwidth}
    \centering
    \vspace{-0.18in}
    \caption{Accuracy (\%) of continuations from low-entropy windows of failed responses, without and with a reflection cue.}
    \label{tab:recovery_intervention}
    \vspace{-0.1in}
    \small
    \setlength{\tabcolsep}{2pt}
    \begin{tabular*}{\linewidth}{@{\extracolsep{\fill}}lccc@{}}
        \toprule
        \textbf{Model} & \textbf{Resampled} & \textbf{Intervened} & \textbf{$\Delta$ (pp)} \\
        \midrule
        Qwen3-4B-Base & 15.9 & 25.8 & $+$10.0 \\
        Qwen3-8B-Base & 10.5 & 17.4 & $+$6.9 \\
        \bottomrule
    \end{tabular*}
    \vspace{-0.1in}
\end{wraptable}

%% file: tables/coupling_8b_recurrence.tex
\begin{table}[h]
    \centering
    \begin{minipage}[t]{0.64\linewidth}
        \caption{\textbf{Sign--entropy coupling on Qwen3-8B-Base.} Macro average of avg@32 / pass@32 (\%) across six benchmarks. Our default \EAPO{} configuration is {\setlength{\fboxsep}{0.5pt}\colorbox{ours!15}{shaded}}. Best scores are \textbf{bolded}.}
        \label{tab:sign_entropy_coupling_8b}
        \vspace{-0.1in}
        \footnotesize
        \setlength{\tabcolsep}{1pt}
        \renewcommand{\arraystretch}{1.25}
        \begin{tabular*}{\linewidth}{@{\extracolsep{\fill}}clccc@{}}
            \toprule
            & \textbf{Entropy} & \multicolumn{3}{c}{\textbf{\textcolor{failure}{Penalization}} ($\hat{A}^i<0$): $b_{-}$} \\
            & \textbf{Preference}
            & Low ($-1$) & Uniform ($0$) & High ($+1$) \\
            \midrule
            \multirow{3}{*}{\shortstack{\textbf{\textcolor{success}{Reinforcement}}\\($\hat{A}^i>0$): $b_{+}$}}
            & High ($+1$)
            & \cellcolor{ours!15}\textbf{34.02\,/\,65.55}
            & 32.14\,/\,63.62
            & 29.77\,/\,52.40 \\
            & Uniform ($0$)
            & 30.55\,/\,59.52
            & 26.78\,/\,52.18
            & 27.48\,/\,52.56 \\
            & Low ($-1$)
            & 27.80\,/\,54.68
            & 25.26\,/\,50.27
            & 24.63\,/\,49.82 \\
            \bottomrule
        \end{tabular*}
    \end{minipage}\hfill
    \begin{minipage}[t]{0.33\linewidth}
        \caption{\textbf{Strategy recurrence} (\%) of initial high-entropy successes and low-entropy failures.}
        \label{tab:strategy_recurrence}
        \vspace{-0.1in}
        \footnotesize
        \setlength{\tabcolsep}{1pt}
        \renewcommand{\arraystretch}{0.89}
        \begin{tabular}{p{\dimexpr0.37\linewidth-2\tabcolsep\relax}>{\centering\arraybackslash}p{\dimexpr0.30\linewidth-2\tabcolsep\relax}>{\centering\arraybackslash}p{\dimexpr0.33\linewidth-2\tabcolsep\relax}}
            \toprule
            \textbf{Method} & \textbf{Success} ($\uparrow$) & \textbf{Failure} ($\downarrow$) \\
            \midrule
            GRPO       & 35.3 & 34.0 \\
            EntropyAdv & 37.3 & 33.9 \\
            HAPO       & 37.8 & 34.2 \\
            80/20      & 37.7 & 36.1 \\
            RLRT       & 38.6 & 29.3 \\
            \rowcolor{ours!15}
            \textbf{\EAPO{}} & \textbf{38.8} & \textbf{23.6} \\
            \bottomrule
        \end{tabular}
    \end{minipage}
\end{table}

%% file: tables/strategy_within_problem.tex
\begin{table}[h]
    \centering
    \caption{\textbf{Average number of distinct strategies} among 32 correct responses per problem with Qwen3-4B-Base on AIME24/25/26 and HMMT26. \textbf{Bold} denotes the best result.}
    \label{tab:strategy_within_problem}
    \vspace{-0.1in}
    \small
    \setlength{\tabcolsep}{7pt}
    \begin{tabular}{lccccccc}
        \toprule
        \textbf{Method} & Initial & GRPO & EntropyAdv & HAPO & 80/20 & RLRT & \textbf{\EAPO{}} \\
        \midrule
        \textbf{\# Strategies} & 1.86 & 1.74 & 1.91 & 1.85 & 1.90 & 2.10 & \textbf{2.42} \\
        \bottomrule
    \end{tabular}
\end{table}

%% file: tables/lora_vs_full.tex
\begin{table}[t]
    \centering
    \caption{LoRA vs. full fine-tuning on Qwen3-4B-Base across mathematical reasoning benchmarks.}
    \label{tab:lora_vs_full}
    \vspace{-0.1in}
    \small
    \setlength{\tabcolsep}{2.5pt}
    \resizebox{\linewidth}{!}{%
    \begin{tabular}{>{\raggedright\arraybackslash}p{2.0cm}>{\raggedright\arraybackslash}p{1.0cm}*{12}{>{\centering\arraybackslash}p{0.9cm}}}
        \toprule
        \multirow{2}{*}{\textbf{Method}}
        & \multirow{2}{*}{\textbf{Mode}}
        & \multicolumn{2}{c}{\textbf{AIME24}}
        & \multicolumn{2}{c}{\textbf{AIME25}}
        & \multicolumn{2}{c}{\textbf{AIME26}}
        & \multicolumn{2}{c}{\textbf{HMMT26}}
        & \multicolumn{2}{c}{\textbf{AMC23}}
        & \multicolumn{2}{c}{\textbf{MATH500-H}} \\
        \cmidrule(lr){3-4}\cmidrule(lr){5-6}\cmidrule(lr){7-8}
        \cmidrule(lr){9-10}\cmidrule(lr){11-12}\cmidrule(lr){13-14}
        & & {\scriptsize Avg@32} & {\scriptsize Pass@32}
        & {\scriptsize Avg@32} & {\scriptsize Pass@32}
        & {\scriptsize Avg@32} & {\scriptsize Pass@32}
        & {\scriptsize Avg@32} & {\scriptsize Pass@32}
        & {\scriptsize Avg@32} & {\scriptsize Pass@32}
        & {\scriptsize Avg@32} & {\scriptsize Pass@32} \\
        \midrule
        \multirow{2}{*}{GRPO}
            & LoRA & 12.5 & 33.3 & 10.4 & 36.7 & 7.5 & 33.3 & 5.9 & 24.2 & \textbf{54.5} & 90.0 & 56.1 & 85.1 \\
            & Full & \textbf{13.1} & \textbf{36.7} & \textbf{12.8} & 36.7 & \textbf{10.2} & 33.3 & \textbf{6.9} & \textbf{27.3} & 53.7 & \textbf{92.5} & \textbf{58.2} & \textbf{87.3} \\
        \midrule
        \multirow{2}{*}{EntropyAdv}
            & LoRA & \textbf{13.1} & \textbf{46.7} & \textbf{12.5} & 36.7 & 10.0 & 33.3 & 6.3 & 27.3 & \textbf{53.8} & \textbf{90.0} & 56.7 & \textbf{86.6} \\
            & Full & 13.0 & 36.7 & 11.0 & 36.7 & 10.0 & \textbf{36.7} & \textbf{7.7} & \textbf{30.3} & 52.5 & 87.5 & \textbf{57.2} & 84.3 \\
        \midrule
        \multirow{2}{*}{HAPO}
            & LoRA & \textbf{12.2} & 36.7 & \textbf{11.0} & 40.0 & 8.9 & 30.0 & \textbf{6.5} & \textbf{24.2} & 52.0 & 87.5 & 53.2 & \textbf{85.1} \\
            & Full & 11.9 & 36.7 & 10.5 & 40.0 & 8.9 & 30.0 & 6.3 & 21.2 & \textbf{55.3} & \textbf{92.5} & \textbf{54.9} & 82.8 \\
        \midrule
        \multirow{2}{*}{80/20}
            & LoRA & \textbf{11.9} & 33.3 & \textbf{12.7} & 40.0 & \textbf{8.6} & 30.0 & \textbf{6.6} & 24.2 & 54.1 & \textbf{95.0} & 52.8 & 84.3 \\
            & Full & 11.6 & 33.3 & 9.2 & 40.0 & 6.9 & 30.0 & 6.1 & \textbf{27.3} & \textbf{55.3} & 90.0 & \textbf{53.0} & \textbf{87.3} \\
        \midrule
        \multirow{2}{*}{RLRT}
            & LoRA & 15.3 & 50.0 & 12.7 & 33.3 & \textbf{14.2} & \textbf{40.0} & \textbf{8.9} & \textbf{36.4} & 54.9 & \textbf{95.0} & 57.4 & 86.6 \\
            & Full & \textbf{19.1} & \textbf{56.7} & \textbf{13.9} & \textbf{36.7} & 14.0 & 36.7 & 8.5 & 30.3 & \textbf{57.0} & 92.5 & \textbf{59.7} & \textbf{88.1} \\
        \midrule
        \rowcolor{ours!15}
            & LoRA & \textbf{20.2} & \textbf{53.3} & 17.1 & \textbf{46.7} & \textbf{15.7} & \textbf{50.0} & \textbf{13.4} & 33.3 & 56.3 & 92.5 & \textbf{63.3} & 91.8 \\
        \rowcolor{ours!15}
        \multirow{-2}{*}{\textbf{\EAPO{} (Ours)}}
            & Full & 18.3 & 50.0 & \textbf{18.2} & 43.3 & 15.2 & 46.7 & 12.3 & \textbf{36.4} & \textbf{57.7} & 92.5 & 62.8 & 91.8 \\
        \bottomrule
    \end{tabular}%
    }
\end{table}

%% file: figures/random_seeds.tex
\begin{figure}[t!]
    \centering
    \includegraphics[width=\linewidth]{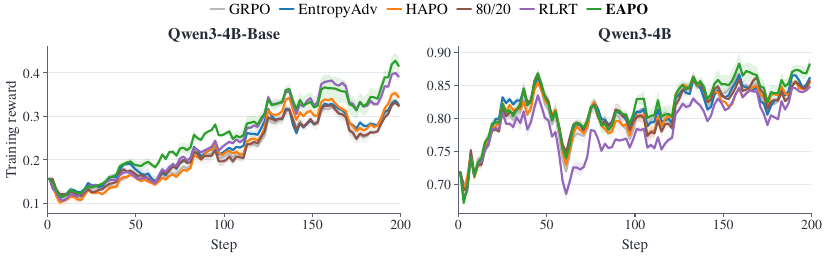}
    \vspace{-0.2in}
    \caption{Training reward on \figleft{} Qwen3-4B-Base and \figright{} Qwen3-4B across three seeds (mean $\pm$ standard deviation), using a trailing 15-step moving average.}
    \label{fig:random_seeds}
\end{figure}

%% file: tables/random_seeds_train.tex
\begin{table}[t]
    \centering
    \caption{Results across three training runs with Qwen3-4B-Base on four mathematical reasoning benchmarks. We report the mean and standard deviation across runs. \textbf{Bold} denotes the best performance.}
    \label{tab:random_seeds_train}
    \vspace{-0.1in}
    \small
    \setlength{\tabcolsep}{2.5pt}
    \resizebox{\linewidth}{!}{%
    \begin{tabular}{p{3.1cm}*{8}{>{\centering\arraybackslash}p{1.45cm}}}
        \toprule
        \multirow{2}{*}{\textbf{Method}}
        & \multicolumn{2}{c}{\textbf{AIME24}}
        & \multicolumn{2}{c}{\textbf{AIME25}}
        & \multicolumn{2}{c}{\textbf{AIME26}}
        & \multicolumn{2}{c}{\textbf{HMMT26}} \\
        \cmidrule(lr){2-3}\cmidrule(lr){4-5}\cmidrule(lr){6-7}\cmidrule(lr){8-9}
        & {\scriptsize Avg@32} & {\scriptsize Pass@32}
        & {\scriptsize Avg@32} & {\scriptsize Pass@32}
        & {\scriptsize Avg@32} & {\scriptsize Pass@32}
        & {\scriptsize Avg@32} & {\scriptsize Pass@32} \\
        \midrule
        GRPO       & \ms{12.8}{0.7} & \ms{38.9}{7.7} & \ms{12.7}{0.7} & \ms{41.1}{1.9} & \ms{9.2}{1.7} & \ms{34.4}{5.1} & \ms{5.9}{0.9} & \ms{20.2}{1.7} \\
        EntropyAdv & \ms{12.3}{0.8} & \ms{42.2}{1.9} & \ms{10.5}{1.2} & \ms{41.1}{1.9} & \ms{8.2}{0.4} & \ms{34.4}{5.1} & \ms{5.1}{1.1} & \ms{28.3}{1.7} \\
        HAPO       & \ms{13.1}{0.8} & \ms{41.1}{3.8} & \ms{11.6}{0.8} & \ms{36.7}{3.3} & \ms{8.6}{0.7} & \ms{25.6}{6.9} & \ms{6.5}{0.2} & \ms{25.3}{3.5} \\
        80/20      & \ms{11.4}{1.1} & \ms{37.8}{5.1} & \ms{12.9}{0.5} & \ms{34.4}{1.9} & \ms{6.3}{0.3} & \ms{33.3}{8.8} & \ms{5.9}{0.9} & \ms{23.2}{3.5} \\
        RLRT       & \ms{13.2}{0.7} & \ms{47.8}{3.8} & \ms{14.0}{1.0} & \ms{36.7}{0.0} & \ms{13.9}{0.7} & \ms{40.0}{0.0} & \ms{7.7}{0.7} & \ms{25.3}{4.6} \\
        \rowcolor{ours!15}
        \textbf{\EAPO{} (Ours)}
                   & \ms{\textbf{20.2}}{0.7} & \ms{\textbf{53.3}}{0.0} & \ms{\textbf{17.4}}{0.8} & \ms{\textbf{43.3}}{3.3} & \ms{\textbf{16.2}}{1.4} & \ms{\textbf{43.3}}{8.8} & \ms{\textbf{11.4}}{0.6} & \ms{\textbf{31.3}}{3.5} \\
        \bottomrule
    \end{tabular}%
    }
\end{table}

%% file: tables/random_seeds_eval.tex
\begin{table}[t]
    \centering
    \caption{Results across three evaluation seeds with base and reasoning models on four mathematical reasoning benchmarks. We report the mean and standard deviation across seeds. \textbf{Bold} denotes the best performance within each backbone.}
    \label{tab:random_seeds_eval}
    \vspace{-0.1in}
    \small
    \setlength{\tabcolsep}{2.5pt}
    \resizebox{\linewidth}{!}{%
    \begin{tabular}{p{3.1cm}*{8}{>{\centering\arraybackslash}p{1.45cm}}}
        \toprule
        \multirow{2}{*}{\textbf{Method}}
        & \multicolumn{2}{c}{\textbf{AIME24}}
        & \multicolumn{2}{c}{\textbf{AIME25}}
        & \multicolumn{2}{c}{\textbf{AIME26}}
        & \multicolumn{2}{c}{\textbf{HMMT26}} \\
        \cmidrule(lr){2-3}\cmidrule(lr){4-5}\cmidrule(lr){6-7}\cmidrule(lr){8-9}
        & {\scriptsize Avg@32} & {\scriptsize Pass@32}
        & {\scriptsize Avg@32} & {\scriptsize Pass@32}
        & {\scriptsize Avg@32} & {\scriptsize Pass@32}
        & {\scriptsize Avg@32} & {\scriptsize Pass@32} \\
        \midrule
        \textit{Qwen3-4B-Base}
                   & \ms{10.7}{0.2} & \ms{44.4}{1.9} & \ms{5.7}{0.4} & \ms{28.9}{1.9} & \ms{6.6}{0.2} & \ms{23.3}{0.0} & \ms{3.9}{0.0} & \ms{21.2}{3.0} \\
        \quad GRPO       & \ms{12.0}{0.4} & \ms{38.9}{3.8} & \ms{12.8}{0.6} & \ms{42.2}{6.9} & \ms{9.1}{0.5} & \ms{33.3}{3.3} & \ms{5.9}{0.2} & \ms{22.2}{1.7} \\
        \quad EntropyAdv & \ms{12.9}{0.4} & \ms{43.3}{3.3} & \ms{11.4}{0.3} & \ms{36.7}{0.0} & \ms{8.5}{0.7} & \ms{35.6}{5.1} & \ms{5.9}{0.3} & \ms{27.3}{3.0} \\
        \quad HAPO       & \ms{13.2}{1.1} & \ms{38.9}{1.9} & \ms{11.9}{0.2} & \ms{37.8}{5.1} & \ms{8.9}{0.3} & \ms{26.7}{3.3} & \ms{6.1}{0.4} & \ms{29.3}{1.7} \\
        \quad 80/20      & \ms{10.9}{0.3} & \ms{35.6}{5.1} & \ms{11.4}{0.3} & \ms{36.7}{3.3} & \ms{7.3}{0.3} & \ms{30.0}{3.3} & \ms{6.8}{0.2} & \ms{25.3}{1.7} \\
        \quad RLRT       & \ms{14.5}{0.6} & \ms{48.9}{5.1} & \ms{13.9}{0.6} & \ms{38.9}{6.9} & \ms{13.6}{0.6} & \ms{40.0}{5.8} & \ms{7.4}{0.2} & \ms{28.3}{4.6} \\
        \rowcolor{ours!15}
        \quad \textbf{\EAPO{} (Ours)}
                   & \ms{\textbf{20.5}}{0.4} & \ms{\textbf{51.1}}{6.9} & \ms{\textbf{17.1}}{0.5} & \ms{\textbf{45.6}}{1.9} & \ms{\textbf{16.2}}{0.6} & \ms{\textbf{44.4}}{3.8} & \ms{\textbf{12.0}}{0.2} & \ms{\textbf{35.4}}{6.3} \\
        \midrule

        \textit{Qwen3-8B-Base}
                   & \ms{13.2}{0.5} & \ms{41.1}{1.9} & \ms{13.1}{0.4} & \ms{42.2}{1.9} & \ms{7.6}{0.3} & \ms{33.3}{0.0} & \ms{5.7}{0.1} & \ms{24.2}{0.0} \\
        \quad GRPO       & \ms{16.1}{0.3} & \ms{40.0}{3.3} & \ms{14.1}{0.7} & \ms{35.6}{1.9} & \ms{11.1}{0.8} & \ms{32.2}{1.9} & \ms{6.7}{0.4} & \ms{27.3}{0.0} \\
        \quad EntropyAdv & \ms{15.9}{0.9} & \ms{43.3}{5.8} & \ms{15.2}{0.3} & \ms{42.2}{1.9} & \ms{8.8}{0.9} & \ms{35.6}{5.1} & \ms{6.1}{0.9} & \ms{23.2}{1.7} \\
        \quad HAPO       & \ms{14.9}{0.7} & \ms{37.8}{1.9} & \ms{13.6}{1.2} & \ms{31.1}{1.9} & \ms{8.6}{0.2} & \ms{37.8}{3.8} & \ms{5.1}{0.7} & \ms{24.2}{0.0} \\
        \quad 80/20      & \ms{16.6}{0.3} & \ms{43.3}{3.3} & \ms{14.8}{0.5} & \ms{38.9}{6.9} & \ms{13.1}{0.2} & \ms{36.7}{0.0} & \ms{7.1}{0.3} & \ms{30.3}{6.1} \\
        \quad RLRT       & \ms{22.8}{0.3} & \ms{\textbf{58.9}}{3.8} & \ms{18.8}{0.4} & \ms{45.6}{1.9} & \ms{17.1}{0.6} & \ms{48.9}{5.1} & \ms{11.8}{0.3} & \ms{33.3}{3.0} \\
        \rowcolor{ours!15}
        \quad \textbf{\EAPO{} (Ours)}
                   & \ms{\textbf{23.9}}{0.4} & \ms{57.8}{5.1} & \ms{\textbf{20.7}}{0.4} & \ms{\textbf{48.9}}{1.9} & \ms{\textbf{17.6}}{1.7} & \ms{\textbf{52.2}}{3.8} & \ms{\textbf{12.5}}{0.8} & \ms{\textbf{35.4}}{3.5} \\
        \midrule

        \textit{Qwen3-4B}
                   & \ms{69.3}{1.7} & \ms{84.4}{1.9} & \ms{59.8}{0.9} & \ms{84.4}{1.9} & \ms{61.1}{1.1} & \ms{85.6}{1.9} & \ms{40.6}{0.1} & \ms{66.7}{3.0} \\
        \quad GRPO       & \ms{70.7}{0.7} & \ms{87.8}{1.9} & \ms{63.4}{1.0} & \ms{85.6}{1.9} & \ms{62.5}{0.5} & \ms{83.3}{0.0} & \ms{40.5}{0.6} & \ms{69.7}{3.0} \\
        \quad EntropyAdv & \ms{72.2}{0.6} & \ms{87.8}{3.8} & \ms{62.7}{0.6} & \ms{84.4}{5.1} & \ms{63.8}{0.4} & \ms{84.4}{1.9} & \ms{42.6}{0.8} & \ms{70.7}{1.7} \\
        \quad HAPO       & \ms{71.2}{1.3} & \ms{88.9}{1.9} & \ms{62.7}{0.5} & \ms{83.3}{3.3} & \ms{62.3}{0.8} & \ms{84.4}{1.9} & \ms{41.6}{0.7} & \ms{68.7}{1.7} \\
        \quad 80/20      & \ms{69.6}{0.3} & \ms{87.8}{3.8} & \ms{61.6}{0.3} & \ms{\textbf{87.8}}{1.9} & \ms{59.6}{0.5} & \ms{83.3}{3.3} & \ms{41.0}{0.1} & \ms{68.7}{4.6} \\
        \quad RLRT       & \ms{\textbf{73.8}}{0.8} & \ms{88.9}{1.9} & \ms{63.3}{0.4} & \ms{\textbf{87.8}}{1.9} & \ms{63.2}{0.1} & \ms{84.4}{1.9} & \ms{42.3}{0.8} & \ms{72.7}{3.0} \\
        \rowcolor{ours!15}
        \quad \textbf{\EAPO{} (Ours)}
                   & \ms{72.5}{0.8} & \ms{\textbf{90.0}}{0.0} & \ms{\textbf{64.4}}{1.2} & \ms{84.4}{5.1} & \ms{\textbf{66.1}}{0.6} & \ms{\textbf{86.7}}{0.0} & \ms{\textbf{44.1}}{0.5} & \ms{\textbf{73.7}}{1.7} \\
        \midrule

        \textit{Olmo-3-7B-Think-DPO}
                   & \ms{74.5}{0.8} & \ms{\textbf{93.3}}{0.0} & \ms{63.8}{1.5} & \ms{83.3}{0.0} & \ms{68.7}{1.5} & \ms{\textbf{90.0}}{0.0} & \ms{44.9}{0.6} & \ms{72.7}{3.0} \\
        \quad GRPO       & \ms{75.6}{0.2} & \ms{\textbf{93.3}}{0.0} & \ms{64.3}{0.8} & \ms{85.6}{1.9} & \ms{69.6}{1.6} & \ms{\textbf{90.0}}{0.0} & \ms{46.1}{0.8} & \ms{69.7}{3.0} \\
        \quad EntropyAdv & \ms{75.4}{0.6} & \ms{\textbf{93.3}}{0.0} & \ms{64.7}{0.7} & \ms{85.6}{1.9} & \ms{69.5}{1.0} & \ms{\textbf{90.0}}{0.0} & \ms{45.5}{0.2} & \ms{75.8}{6.1} \\
        \quad HAPO       & \ms{75.8}{1.0} & \ms{\textbf{93.3}}{0.0} & \ms{63.0}{0.5} & \ms{85.6}{3.8} & \ms{66.1}{1.5} & \ms{\textbf{90.0}}{0.0} & \ms{45.5}{0.7} & \ms{71.7}{1.7} \\
        \quad 80/20      & \ms{75.2}{0.5} & \ms{92.2}{1.9} & \ms{65.1}{0.7} & \ms{83.3}{0.0} & \ms{69.5}{0.8} & \ms{88.9}{1.9} & \ms{45.5}{0.7} & \ms{75.8}{0.0} \\
        \quad RLRT       & \ms{\textbf{76.8}}{0.5} & \ms{\textbf{93.3}}{0.0} & \ms{63.4}{0.2} & \ms{85.6}{3.8} & \ms{67.6}{0.4} & \ms{87.8}{1.9} & \ms{44.1}{0.5} & \ms{67.7}{3.5} \\
        \rowcolor{ours!15}
        \quad \textbf{\EAPO{} (Ours)}
                   & \ms{76.5}{1.0} & \ms{\textbf{93.3}}{0.0} & \ms{\textbf{65.2}}{0.8} & \ms{\textbf{88.9}}{1.9} & \ms{\textbf{70.7}}{0.7} & \ms{\textbf{90.0}}{0.0} & \ms{\textbf{46.4}}{0.6} & \ms{\textbf{78.8}}{5.2} \\
        \bottomrule
    \end{tabular}%
    }
\end{table}

%% file: figures/passk_other_benchmarks.tex
\begin{figure}[h!]
    \centering
    \includegraphics[width=\linewidth]{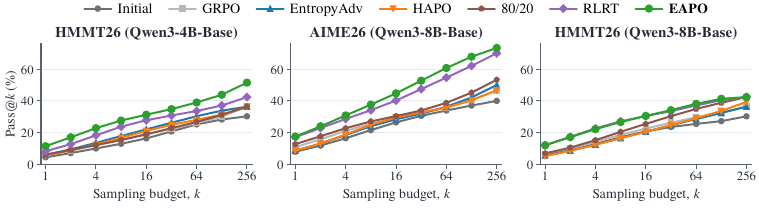}
    \vspace{-0.2in}
    \caption{Pass@$k$ across sampling budgets on additional backbones and benchmarks.}
    \label{fig:passk_other_benchmarks}
\end{figure}

%% file: sections/appendix/4_theoretical_analysis.tex
\clearpage
\section{Proofs and Additional Theoretical Analysis}
\label{app:theoretical_analysis}

We provide the proofs of \Cref{prop:kl_allocation,prop:alternative_concentration} (\Cref{app:kl_allocation_proof,app:concentration_proof}) and present an information-theoretic motivation for entropy-guided allocation under a binary decision model (\Cref{app:allocation_implications}).

\subsection{KL-Regularized Uncertainty Allocation}
\label{app:kl_allocation_proof}

The exponential allocation in \Cref{eq:entropic_allocation} is motivated by extending KL anchoring from policy optimization~\citep{grpo} to token-level credit allocation, favoring greater signed uncertainty while anchoring credit to a uniform reference. We formalize this motivation through \Cref{eq:kl_allocation}, where $u_i$ is the reference distribution and $\sign(\hat{A}^i)h_{i,t}$ is the utility. The resulting optimization is a finite-dimensional instance of the Gibbs variational principle~\citep{donsker-varadhan}.

\noindent\textbf{\Cref{prop:kl_allocation} (Restated).}
\begingroup\itshape
Let $u_{i,t}=1/T_i$ be uniform credit over $\mathcal{I}_i$, and let $\Delta(\mathcal{I}_i)$ denote the probability simplex on these positions. For $\kappa>0$, the problem
\begin{equation}
    \max_{q\in\Delta(\mathcal{I}_i)}
    \left\{
        \sign(\hat{A}^i)\E_{t\sim q}[h_{i,t}]
        -\frac{1}{\kappa}\KL(q\Vert u_i)
    \right\}
    \label{eq:kl_allocation_restated}
\end{equation}
has the unique solution $q^*_{i,t}\propto u_{i,t}\exp\!\left(\kappa\sign(\hat{A}^i)h_{i,t}\right)$.
\par\endgroup

\begin{proof}
Let $s_i=\sign(\hat{A}^i)$ and define
\begin{equation}
    Z_i=\sum_{t\in\mathcal{I}_i}u_{i,t}\exp\!\left(\kappa s_i h_{i,t}\right),
    \qquad
    q_{i,t}^*=\frac{u_{i,t}\exp\!\left(\kappa s_i h_{i,t}\right)}{Z_i}.
    \label{eq:kl_allocation_solution}
\end{equation}
For any $q\in\Delta(\mathcal{I}_i)$, the objective in \Cref{eq:kl_allocation_restated} can be written as
\begin{equation}
    s_i\E_{t\sim q}[h_{i,t}]-\frac{1}{\kappa}\KL(q\Vert u_i)
    =\frac{\log Z_i}{\kappa}-\frac{1}{\kappa}\KL(q\Vert q^*).
    \label{eq:kl_allocation_decomposition}
\end{equation}
Since $\KL(q\Vert q^*)\ge0$, with equality if and only if $q=q^*$, the allocation $q^*$ is the unique global maximizer. Moreover, $u_{i,t}=1/T_i$ and \Cref{eq:entropic_allocation} give $q_{i,t}^*=w_{i,t}/T_i$.
\end{proof}

\subsection{Concentration of Unchosen Alternatives}
\label{app:concentration_proof}

We examine how stronger negative updates change the distribution over unchosen tokens, providing a theoretical rationale for \EAPO{}'s penalty attenuation. At a fixed prefix, let $\vz$ be the initial logits, $\vp=\softmax(\vz)$ the corresponding token distribution, and $a$ the sampled token. For fixed $\hat{A}<0$ and allocation weight $w$, we consider one policy-gradient step on $\vz$:
\begin{equation}
    \vz^{(\beta)}=\vz-\beta(\ve_a-\vp),
    \qquad \beta=\eta|\hat{A}|w\ge0,
    \qquad \vp^{(\beta)}=\softmax(\vz^{(\beta)}),
    \label{eq:negative_logit_step}
\end{equation}
where $\ve_a$ is the one-hot vector for $a$, $\eta>0$ is the effective step size, and $\beta$ controls the strength of the negative update. We exclude the sampled token and renormalize the remaining probabilities to obtain the conditional distribution over unchosen tokens:
\begin{equation}
    \nu_v^{(\beta)}
    =\frac{p_v^{(\beta)}}{1-p_a^{(\beta)}},
    \qquad v\in\mathcal{V}\setminus\{a\}.
\end{equation}
Here, $p_v^{(\beta)}$ is the probability of token $v$ after the update and $\nu=\nu^{(0)}$ is the initial conditional distribution. The entropy of $\nu^{(\beta)}$ measures concentration among the alternatives, while $\KL(\nu^{(\beta)}\Vert\nu)$ measures their deviation from the initial proportions.

\clearpage
\noindent\textbf{\Cref{prop:alternative_concentration} (Restated).}
\begingroup\itshape
For a single negative logit update with finite logits and the initial distribution held fixed, the sampled-token probability strictly decreases with $\beta$, while $H(\nu^{(\beta)})$ is nonincreasing and $\KL(\nu^{(\beta)}\Vert\nu)$ is nondecreasing.
\par\endgroup

\begin{proof}
Let us assume finite logits over a finite vocabulary with $|\mathcal{V}|\ge2$, so all initial and updated probabilities are positive and $0<p_a^{(\beta)}<1$. For any $v\ne a$, subtracting the updated logits gives
\begin{equation}
    \log\frac{p_a^{(\beta)}}{p_v^{(\beta)}}
    =\log\frac{p_a}{p_v}-\beta(1-p_a+p_v).
    \label{eq:negative_log_odds}
\end{equation}
Every coefficient $1-p_a+p_v$ is positive, so each ratio $p_v^{(\beta)}/p_a^{(\beta)}$ strictly increases with $\beta$. Therefore, $p_a^{(\beta)}=(1+\sum_{v\ne a}p_v^{(\beta)}/p_a^{(\beta)})^{-1}$ establishes strict decrease of the sampled-token probability.

For unchosen tokens, the update adds $\beta p_v$ to each logit. Conditioning on this set cancels the full softmax normalizer, yielding
\begin{equation}
    \nu_v^{(\beta)}
    =\frac{\nu_v e^{\beta p_v}}{Z_a(\beta)},
    \qquad
    Z_a(\beta)=\sum_{u\ne a}\nu_u e^{\beta p_u}.
    \label{eq:alternative_tilt}
\end{equation}
Let $\psi_a(\beta)=\log Z_a(\beta)$. Differentiation gives $\psi_a'(\beta)=\E_{v\sim\nu^{(\beta)}}[p_v]$ and $\psi_a''(\beta)=\Var_{v\sim\nu^{(\beta)}}(p_v)$. The same calculation gives $\frac{d}{d\beta}\E_{\nu^{(\beta)}}[\log\nu_v]=\Cov_{\nu^{(\beta)}}(\log\nu_v,p_v)$.

Writing $H(\nu^{(\beta)})=-\E_{\nu^{(\beta)}}[\log\nu_v]-\beta\psi_a'(\beta)+\psi_a(\beta)$, we obtain
\begin{equation}
    \frac{d}{d\beta}H(\nu^{(\beta)})
    =-\Cov_{\nu^{(\beta)}}(\log\nu_v,p_v)
    -\beta\Var_{\nu^{(\beta)}}(p_v)
    \le0.
    \label{eq:alternative_entropy_derivative}
\end{equation}
To verify the inequality, let $V,U$ be independent draws from $\nu^{(\beta)}$. Then
\begin{equation}
    \Cov_{\nu^{(\beta)}}(\log\nu_v,p_v)
    =\frac{1}{2}\E\!\left[
        (\log\nu_V-\log\nu_U)(p_V-p_U)
    \right]\ge0,
\end{equation}
because $\log\nu_v=\log p_v-\log(1-p_a)$ is increasing in $p_v$.

Similarly, \Cref{eq:alternative_tilt} implies
\begin{equation}
    \KL(\nu^{(\beta)}\Vert\nu)
    =\beta\psi_a'(\beta)-\psi_a(\beta),
    \qquad
    \frac{d}{d\beta}\KL(\nu^{(\beta)}\Vert\nu)
    =\beta\Var_{\nu^{(\beta)}}(p_v)\ge0.
    \label{eq:alternative_kl_derivative}
\end{equation}
The entropy and KL remain unchanged when all alternatives have the same initial probability.
\end{proof}

\subsection{Implications for Entropy-Guided Allocation}
\label{app:allocation_implications}

We further examine how feedback information varies with entropy under success and failure, providing an information-theoretic interpretation of \EAPO{}'s asymmetric credit allocation. Consider a binary decision with correct branch $B\in\{0,1\}$ and belief $P_p(B=1)=p\in(0,1)$. We assume that the policy samples $a$ independently of $B$, with $\Pr(a=1)=p$, and receives feedback $R=\mathbf{1}\{a=B\}$. We define \emph{feedback information} as the expected posterior correction conditional on the outcome:
\begin{equation}
    J_r(p)=\E\!\left[
        \KL\!\left(P_p(B\mid a,R)\Vert P_p(B)\right)
        \,\middle|\, R=r
    \right],
    \qquad r\in\{0,1\}.
    \label{eq:conditional_feedback_information}
\end{equation}

\begin{proposition}[Asymmetric information from success and failure]
\label{prop:feedback_information}
For the binary decision model, $J_1(p)$ is strictly increasing in the binary entropy $H(p)$, whereas $J_0(p)$ is strictly decreasing in $H(p)$.
\end{proposition}

\begin{proof}
Let $D=p^2+(1-p)^2$. The pair $(a,R)$ identifies $B$, so the posterior KL is $-\log P_p(B)$. Conditional on success, the probabilities of $B=1$ and $B=0$ are $p^2/D$ and $(1-p)^2/D$, whereas on failure both are $1/2$ because each mismatched pair has probability $p(1-p)$. Substituting into \Cref{eq:conditional_feedback_information} gives $J_1(p)=-\frac{p^2\log p+(1-p)^2\log(1-p)}{D}$ and $J_0(p)=-\frac{1}{2}\log\!\left(p(1-p)\right)$. By symmetry and continuity, it suffices to consider $p>1/2$, where $H'(p)=-\log\frac{p}{1-p}<0$ and
\begin{equation}
    J_1'(p)=-\frac{2p-1}{D}
        -\frac{2p(1-p)}{D^2}\log\frac{p}{1-p}<0,
    \qquad
    J_0'(p)=\frac{2p-1}{2p(1-p)}>0.
    \label{eq:feedback_information_derivatives}
\end{equation}
\end{proof}

\Cref{prop:feedback_information} shows that success leads to a larger expected correction of the prior belief at uncertain decisions, whereas failure does so at confident decisions. This matches \EAPO{}'s preference for high entropy under success and low entropy under failure.

Uniform allocation gives every decision the same share of credit, while \EAPO{} shifts credit toward decisions with greater expected feedback information. For $T$ decisions with the same outcome $r$, let $J_t=J_r(p_t)$ and let $q_t^{\mathrm{E}}$ denote the fraction of credit assigned to decision $t$ by \Cref{eq:entropic_allocation}. Then
\begin{equation}
    \E_{q^{\mathrm{E}}}[J_t]-\E_u[J_t]
    =\frac{\Cov_u\!\left(J_t,e^{\kappa s h_t}\right)}
        {\E_u[e^{\kappa s h_t}]}\ge0,
    \label{eq:feedback_information_allocation}
\end{equation}
where $u_t=1/T$, $s=2r-1$, and $h_t$ is the normalized entropy from \Cref{eq:normalized_entropy}. The covariance is nonnegative because entropy normalization preserves the ordering and larger credit shares correspond to larger $J_t$. Thus, in this binary model, the average feedback information weighted by \EAPO{} is larger than the uniform average when $\kappa>0$ and the normalized entropies are nonuniform. Favoring high entropy under both outcomes would instead shift credit toward less informative decisions under failure. This supports \EAPO{}'s asymmetric credit allocation by showing that it prioritizes more informative decisions under both outcomes, while \Cref{app:concentration_proof} supports attenuating penalties at uncertain decisions to preserve alternatives for further exploration.

%% file: sections/appendix/5_qualitative_example.tex
\clearpage
\section{Example of Surprising Success, Repeated Failure}
\label{app:qualitative_example_succ_fail}

\Cref{fig:qualitative_success_failure} illustrates \repeatedfailure{} and \surprisingsuccess{} in 32 rollouts from the initial Qwen3-4B-Base model on a MATH500 problem: all 28 AM--GM rollouts fail, whereas one of four Lagrange-multiplier rollouts succeeds. Specifically, the model repeatedly applies AM--GM to $x^4$, $4y^2$, and $4z^4$, resulting in the invalid substitution $x^4y^2z^4=(xyz)^4$. Yet a successful alternative remains within AM--GM: splitting $4y^2$ into $2y^2+2y^2$ allows the same inequality to establish the correct minimum of $16$. Meanwhile, the less frequently attempted Lagrange-multiplier approach also reaches this minimum, illustrating a successful alternative that remains rare under the current policy and should be reinforced to make such success more repeatable.

\input{figures/qualitative_success_failure}

%% file: figures/qualitative_success_failure.tex
\begin{figure}[h]
    \centering
    \begin{minipage}{\linewidth}
        \fontsize{9}{11.5}\selectfont
        \setlength{\parindent}{0pt}
        \setlength{\parskip}{0pt}

        \begin{qualitativebox}[colframe=qualitativeframe,boxrule=0.55pt,top=8pt,bottom=8pt,before skip=0pt]
        \begin{qualitativebox}[colback=qualitativeproblem,boxrule=0pt,top=6pt,bottom=6pt,before skip=0pt]
            {
                \textbf{Problem.} Let $x$, $y$, and $z$ be positive real numbers such that $xyz=2$.
                Find the minimum value of
                \par\vspace{4pt}
                {\centering $\displaystyle x^4+4y^2+4z^4.$\par}
            }
        \end{qualitativebox}

        \begin{qualitativerow}
            \textbf{Sample 0}
            \par\vspace{5pt}
            {\raggedright We will apply the \textcolor{failure}{AM-GM Inequality} to the terms $x^4$, $4y^2$, and $4z^4$:\par}
            \vspace{4pt}
            \qualitativestep{\qualitativeomitted\quad
                x^4+4y^2+4z^4\geq 3\sqrt[3]{16x^4y^2z^4},
                \qquad x^4y^2z^4=(xyz)^4=2^4=16.}
            \qualitativestep{\qualitativeomitted\quad
                x^4+4y^2+4z^4=4\sqrt{2}+4\sqrt{2}+4\sqrt{2}=\textcolor{failure}{\boxed{12}}.}
        \end{qualitativerow}

        \qualitativeseparator
        \begin{qualitativerow}
            \textbf{Sample 4}
            \par\vspace{5pt}
            {\raggedright We start by applying the \textcolor{failure}{AM-GM inequality} to the terms $x^4$, $4y^2$, and $4z^4$:\par}
            \vspace{4pt}
            \qualitativestep{\qualitativeomitted\quad
                x^4+4y^2+4z^4\geq 3\sqrt[3]{16x^4y^2z^4},
                \qquad x^4y^2z^4=(xyz)^4=2^4=16.}
            {\raggedright $[\ldots]$ However, we need to recheck the equality condition and the value of $k$.\par}
            \vspace{4pt}
            \qualitativestep{\qualitativeomitted\quad
                x^4+4y^2+4z^4=2^4+4\cdot 1^2+4\cdot 1^4=16+4+4=\textcolor{failure}{\boxed{24}}.}
        \end{qualitativerow}

        \qualitativeseparator
        \begin{qualitativerow}
            \textbf{Sample 10}
            \par\vspace{5pt}
            {\raggedright To find the minimum value, we need to check if equality can be achieved in the \textcolor{failure}{AM-GM inequality}. Equality holds when all the terms are equal:\par}
            \vspace{4pt}
            \qualitativestep{x^4=4y^2=4z^4.}
            \qualitativestep{\qualitativeomitted\quad
                x^4=4y^2=4z^4=4\sqrt{2},\qquad
                x^4+4y^2+4z^4=3\cdot4\sqrt{2}=\textcolor{failure}{\boxed{12\sqrt{2}}}.}
        \end{qualitativerow}

        \qualitativeseparator
        \begin{qualitativerow}
            \textbf{Sample 5}
            \par\vspace{5pt}
            {\raggedright We will use the method of \textcolor{success}{Lagrange multipliers} to find the minimum value of the function $f(x,y,z)=x^4+4y^2+4z^4$ subject to the constraint $g(x,y,z)=xyz-2=0$.\par}
            \vspace{4pt}
            \qualitativestep{-\frac{4x^3}{yz}=-\frac{8y}{xz}\ \Longrightarrow\ x^4=2y^2,
                \qquad -\frac{4x^3}{yz}=-\frac{16z^3}{xy}\ \Longrightarrow\ x^4=4z^4.}
            \qualitativestep{\qualitativeomitted\quad (x,y,z)=(\sqrt{2},\sqrt{2},1),
                \qquad x^4+4y^2+4z^4=4+8+4=\textcolor{success}{\boxed{16}}.}
        \end{qualitativerow}
        \end{qualitativebox}
    \end{minipage}
    \vspace{-0.05in}
    \caption{\textbf{Qualitative example of \surprisingsuccess{} and \repeatedfailure{}.} Among 32 rollouts of this problem, 28 use the AM--GM inequality and all fail, while four use Lagrange multipliers and only one succeeds. The model repeatedly applies AM--GM and makes the invalid substitution $x^4y^2z^4=(xyz)^4$, yet splitting $4y^2$ into $2y^2+2y^2$ would enable the same approach to reach the correct minimum of $16$. The less frequently sampled Lagrange-multiplier approach yields the sole observed correct response, motivating reinforcement of this rare success to make it more repeatable.}
    \label{fig:qualitative_success_failure}
\end{figure}

%% file: sections/appendix/6_strategy_case.tex
\clearpage
\section{Case Study of Solution Strategies}
\label{app:strategy_case}

To examine whether \EAPO{} escapes confident failures that recur under the initial policy and instead discovers successful strategies that other methods do not explore, we analyze the distribution of solution strategies on a problem from HMMT26 with Qwen3-4B-Base, using the strategy labels described in \Cref{app:strategy_coverage_details}.

\input{figures/strategy_case}

As shown in \Cref{fig:strategy_case}, the initial policy concentrates most of its responses on a single confident but incorrect strategy and repeatedly fails, and the other methods retain this failure mode after training while rarely attempting the strategy that leads to the correct answer. In contrast, \EAPO{} reduces this dominant failing strategy and redistributes its responses across alternative strategies, reaching the correct answer through a strategy that the other methods rarely or never attempt. This illustrates how correcting repeated failure while preserving alternatives at uncertain decisions allows \EAPO{} to escape a dominant failure mode and discover successful strategies that remain unexplored under the other methods.

%% file: figures/strategy_case.tex
\begin{figure}[h]
    \centering
    \begin{minipage}{\linewidth}
        \fontsize{9}{11.5}\selectfont
        \setlength{\parindent}{0pt}
        \setlength{\parskip}{0pt}

        \begin{qualitativebox}[colframe=qualitativeframe,boxrule=0.55pt,top=8pt,bottom=4pt,before skip=0pt]
        \begin{qualitativebox}[colback=qualitativeproblem,boxrule=0pt,top=6pt,bottom=6pt,before skip=0pt]
            \textbf{Problem.} Sarunyu has a stick of length $1$ with one endpoint marked in red. Every minute, he picks one of his sticks uniformly at random and breaks it into two halves of equal length. Compute the expected length of the stick with the red endpoint after $5$ minutes.
        \end{qualitativebox}

        \begin{qualitativerow}
            \setlength{\parskip}{3pt}
            \stratswatch{stratselect}~\textbf{Condition on selection ($1/k$).} Condition on whether the red stick is picked, which happens with probability $1/k$ when there are $k$ sticks, so its expected length shrinks by a factor of $1-\frac{1}{2k}$ each minute: $\prod_{k=1}^{5}\big(1-\frac{1}{2k}\big)=\textcolor{success}{\boxed{63/256}}$.

            \stratswatch{strathalve}~\textbf{Halve every minute.} Assume that the red stick is halved every minute: $(1/2)^5=\textcolor{failure}{\boxed{1/32}}$.

            \stratswatch{stratfixed}~\textbf{Fixed-probability recurrence.} Halve the red stick with probability $1/2$: $(3/4)^5=\textcolor{failure}{\boxed{243/1024}}$.

            \stratswatch{stratsync}~\textbf{Synchronous split.} Assume that all sticks are split every minute: $\textcolor{failure}{\boxed{1/32}}$.

            \stratswatch{stratunchanged}~\textbf{Length unchanged.} Argue that breaking does not change the expected length of the red stick: $\textcolor{failure}{\boxed{1}}$.

            \stratswatch{stratavg}~\textbf{Average length.} Take the average length of the six sticks: $\textcolor{failure}{\boxed{1/6}}$.
        \end{qualitativerow}
        \end{qualitativebox}
    \end{minipage}

    \vspace{0.1in}
    \includegraphics[width=\linewidth]{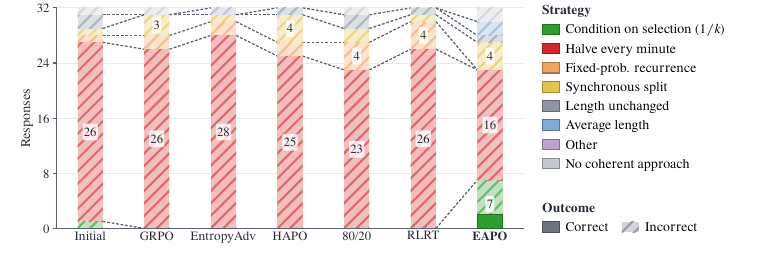}
    \vspace{-0.15in}
    \caption{\textbf{Distribution of solution strategies} on a problem from HMMT26 with Qwen3-4B-Base across 32 responses per method. \figtop{} The problem and the named strategies. \figbottom{} Strategy distribution, where dashed lines connect the same strategy across methods.}
    \label{fig:strategy_case}
\end{figure}